\documentclass[twoside]{article}

\usepackage[accepted]{aistats2025}
\usepackage[utf8]{inputenc} 
\usepackage[T1]{fontenc}    
\usepackage{hyperref}       
\usepackage{url}            
\usepackage{booktabs}       
\usepackage{caption,subcaption,array}
\usepackage{amsfonts}       
\usepackage{nicefrac}       
\usepackage{microtype}      
\usepackage{xcolor}         
\usepackage{enumitem}
\usepackage{cancel}
\usepackage{graphicx,wrapfig}
\usepackage{tabularx}
\usepackage{tikz}
\usepackage{tcolorbox}
\usepackage{colortbl}

\usepackage[round]{natbib}

\definecolor{lb}{rgb}{0.231, 0.459, 0.686}

\tcbset{
    tab2/.style={
    colback=lb!1!white,       
    fonttitle=\bfseries,       
    coltitle=black,            
    halign title=center,       
    }
}

\usepackage{amsmath}
\usepackage{amssymb}
\usepackage{mathtools}
\usepackage{amsthm}

\usepackage{thmtools}

\newcolumntype{L}[1]{>{\raggedright\arraybackslash}p{#1}}
\newcolumntype{C}[1]{>{\centering\arraybackslash}p{#1}}
\newcolumntype{R}[1]{>{\raggedleft\arraybackslash}p{#1}}

\usepackage[capitalize,noabbrev]{cleveref}

\theoremstyle{plain}
\newtheorem{theorem}{Theorem}[section]

\theoremstyle{definition}
\newtheorem{definition}[theorem]{Definition}

\theoremstyle{remark}

\def\s{\mathbf{s}}
\def\R{\mathbb{R}}
\def\a{\mathbf{a}}

\begin{document}

%
\runningtitle{What Ails Generative Structure-based Drug Design?}

%

\twocolumn[

\aistatstitle{What Ails Generative Structure-based Drug Design:\\  Expressivity is Too Little or Too Much?}

\aistatsauthor{ Rafał Karczewski \And Samuel Kaski \And  Markus Heinonen \And Vikas Garg }

\aistatsaddress{ Aalto University \And  Aalto University \\ University of Manchester \And Aalto University \And Aalto University \\ YaiYai Ltd. } ]

\begin{abstract}
Several generative models with elaborate training and sampling procedures have been proposed to accelerate structure-based drug design (SBDD); however, their empirical performance turns out to be suboptimal. We seek to better understand this phenomenon from both theoretical and empirical perspectives. Since most of these models apply graph neural networks (GNNs), one may suspect that they inherit the representational limitations of GNNs. We analyze this aspect, establishing the first such results for protein-ligand complexes. A plausible counterview may attribute the underperformance of these models to their excessive parameterizations, inducing expressivity at the expense of generalization. We investigate this possibility with a simple metric-aware approach that learns an economical surrogate for affinity to infer an unlabelled molecular graph and optimizes for labels conditioned on this graph and molecular properties. The resulting model achieves state-of-the-art results using 100x fewer trainable parameters and affords up to 1000x speedup. Collectively, our findings underscore the need to reassess and redirect the existing paradigm and efforts for SBDD. Code is available at \url{https://github.com/rafalkarczewski/SimpleSBDD}.
\end{abstract}

\section{INTRODUCTION}

 Identifying new molecules or {\em ligands} that bind well to a protein target and have desired properties is a key challenge of Structure-based drug design (SBDD). Since experimental determination of binding is expensive and time-consuming, deep generative models offer hope to accelerate SBDD by rapidly suggesting good candidate molecules for a given target protein, using experimental or \emph{in silico} binding observations \citep{vamathevan2019applications, bilodeau2022generative}.  Several approaches have been proposed, including autoregressive models \citep{pocket2mol, sbddAR}, variational autoencoders \citep{sbddVAE}, reinforcement learning \citep{mcts}, and diffusion models \citep{diffsbdd, target-diff, decompdiff, EQGAT-diff}.

Despite a surge in interest from the machine learning community, these methods are observed to underperform empirically in terms of the docking scores, i.e., the estimated binding affinities of the candidates they generate. These findings are particularly surprising since many of the aforesaid methods jointly optimize for atom types and 3D coordinates (often using GNN-based architectures), with intricate training and sampling procedures. Since the success of SBDD hinges on identifying candidates that dock well \citep{Tobiasz2023}, 
understanding this phenomenon holds key to utilizing the promise of generative modeling in SBDD. 

We take an initial but important step in this pursuit by proposing two contrasting hypotheses that shed light on different aspects of the problem. First, we investigate, theoretically, the possibility that the shortcomings of message-passing GNNs \citep{GNNOrig2009, geometricWL2023},  e.g., due to their inability to compute graph properties \citep{gnns_repr}, might propagate to the protein-ligand complexes, yielding insufficient representations. It turns out that we can expose the limits on the expressivity of GNNs to distinguish ligands conditioned on the same protein target. To our knowledge, these are the first such results in conditional drug discovery contexts, motivating the design and analysis of more effective models that can mitigate the representational limitations.

Second, we also study and analyze, empirically, a counterview that the generative models for SBDD  might already be getting too sophisticated, owing, e.g., to their excessive parameterizations, and simpler methods may provide strong baselines. Some preliminary evidence on the efficacy of simple baselines for the unconditional molecule generation setting, but not SBDD, has emerged recently \citep{tripp2023genetic} stemming from similar concerns. 

Concretely, we identify a key issue that has eluded attention heretofore; namely, that the existing methods predominantly focus on obtaining rich representations for drug properties such as QED with bulky models, while being oblivious to binding affinity as they seek to match the observed ligand distribution  (which often includes molecules with suboptimal affinities). We address these issues with a novel simple two-phase generative method {\em SimpleSBDD} that draws inspiration from the success of performance-aware methods in other domains \citep{joachims05multivariateperformance},  and optimizes for both the metrics of interest, i.e. molecular properties and estimated binding affinity.

Specifically, we first learn an economical affinity surrogate to infer an unlabelled molecular graph, and then optimize its atom labels and coordinates for other properties. Albeit segregating optimization over structure from labels loses some information, we observe that our approach already achieves state of the art in SBDD with up to 1000x speed up and 100x fewer parameters than prior methods. Importantly, we represent the unlabelled molecular graphs with features that are unlearnable by GNNs.

Note that such massive reduction in computation and model size bears significant potential benefits for the docking software in particular, and the overall SBDD pipeline in general. Streamlining the generation of {\em  hits}, i.e., promising initial candidates that bind well to a specified protein target would afford re-channeling of limited resources for {\em lead} optimization as well as more stringent and precise validation stages than docking, namely, design and analysis of simulations for molecular dynamics as well as wet lab experiments. 

\subsection{Contributions of this work}
We summarize our contributions below. This work 
\begin{enumerate}
    \item {\bf (Topical)} draws attention to the disparity between the promise and the observed empirical performance of generative models for SBDD, outlining concerns pertaining to their expressivity;   
    \item{\bf (Theoretical)} initiates a formal analysis for limits on the expressivity of methods for SBDD, exposing the representational challenges inherent in GNN models for encoding protein-ligand complexes;
    \item{\bf (Methodological)} 
    introduces performance-aware optimization in SBDD settings, using a simple generative model that optimizes for both docking scores and physicochemical properties;  and
    \item {\bf (Empirical)} demonstrates the versatility of the proposed approach SimpleSBDD using three instantiations tailored to problems of drug repurposing (suggesting candidates from a known ligand set), generation of novel ligands, and property optimization that seeks to control both molecular properties and binding. SimpleSBDD outperforms existing methods with orders of magnitude fewer parameters and faster runtime.    
\end{enumerate} 


\section{PRELIMINARIES AND RELATED WORK}\label{sec:prelims}
\textbf{Deep (Geometric) Learning for Drug Design}. Deep learning has enabled progress on various facets of drug design across molecular generation \citep{jtvae, moflow, modflow, graphAF, graphDF, edm, difflinker}, molecular optimization \citep{mol_opt} and drug repurposing  \citep{drug_rep}. It has also accelerated advancements in protein folding and design \citep{alphafold, OmegaFold, openfold, esm1, esm2, protseed, rfdiffusion, chroma, foldingdiff, abode} and docking \citep{equibind, equidock, diffdock}. We refer the reader to the surveys by \citet{chen2018rise} and \citet{pandey2022transformational}.

\textbf{Graph Neural Networks (GNNs)} have emerged as a workhorse for modeling graph data, such as molecules \citep{GNNOrig2009, gcn, graphsage, gat}. They can be extended to encode important symmetries in geometric graphs \citep{egnn}, so have found success across diverse tasks in the drug design pipeline including docking \citep{diffdock} and molecular property prediction \citep{egnn_for_prop_pred}.

However, GNNs have known representational limits \citep{xu2018how, Morris2019,  Maron2019,gnns_repr, expressivity-1, expressivity-2, geometricWL2023}. There have been improvements proposed \citep{improving-1, improving-2, improving-3} albeit potentially at the expense of generalization \citep{egnn_generalization}.

\paragraph{Molecular properties}\label{sec:mol_props} 
The goal of SBDD is to generate drug candidates with high binding affinity to a target protein, while controlling for other properties. We list below common molecular properties, calculated in practice using the RDKit software \citep{rdkit}, that are  relevant to drug design \citep{bilodeau2022generative, pocket2mol}: 1) Binding affinity -- the strength of the connection between the ligand and the protein. A common surrogate for affinity is the Vina score \citep{vina}; 2) LogP -- logarithm of the partition coefficient, which is a solubility indicator; 3) QED -- quantitative estimate of drug-likeness, a mixture of properties correlated with drugs; 4) SA -- Synthetic accessibility, an estimate of how easy it would be to synthesize the molecule; and 5) Lipinski's Rule of Five -- A heuristic about whether a molecule would be an active oral drug.

\begin{figure*}[t]
 \centering    
    \includegraphics[width=0.95\textwidth]{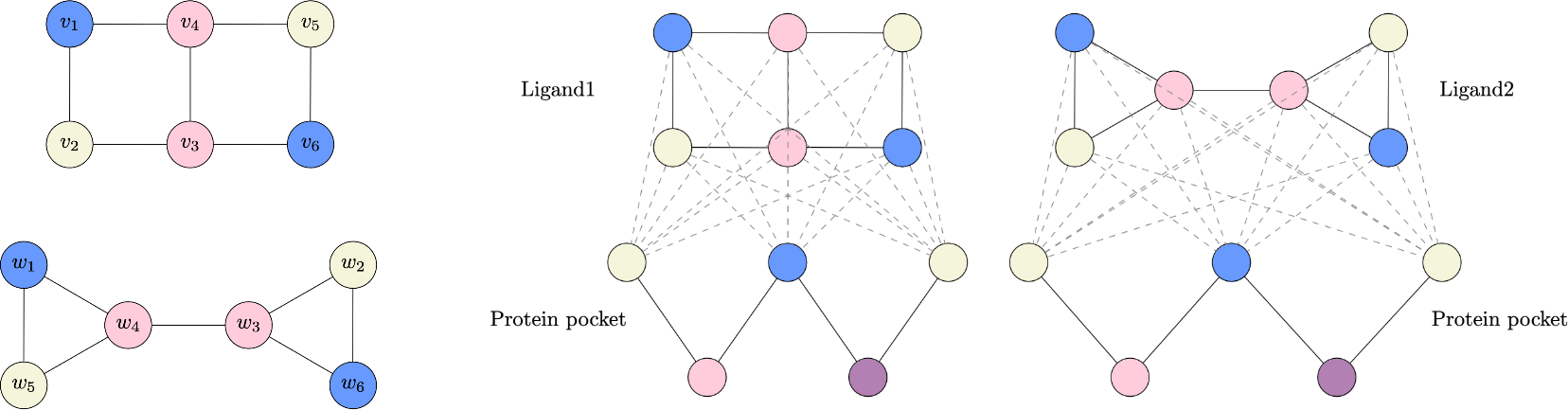}
        \caption{\textbf{Some ligands cannot be distinguished by GNNs even with additional protein context} Left: Construction for Lemma \ref{lemma:3d_single_body}; Two non-isomorphic graphs differing in all properties stated in Lemma \ref{lemma:3d_single_body}, but for which LU-GNNs produce identical embeddings. Right: Complex graphs constructed by joining the ligand graphs with the same protein graph remain identical whenever ligand graphs cannot be differentiated.}
        \label{fig:lu-gnns}
\end{figure*}

\textbf{Setting}.
We view proteins as labelled graphs: each pocket $G = (V,E)$ consists of nodes $V = \{v_1, \ldots, v_N\}$ with $v_i$ representing atoms, and edges $E \subseteq V \times V$. Each node $v = (a, \s)$ is associated with an atom type $a \in \mathcal{A} = \{\texttt{C}, \texttt{N}, \texttt{O}, \ldots \}$ and 3D coordinates $\s \in \R^3$.
We include superscripts $G^P$ and $G^M$ to distinguish between protein and molecule graphs. We also distinguish between a labelled molecular graph $G^M$ and unlabelled one $U^M$. The unlabelled graph contains information about the number of nodes and edges between them, but no information about atom types or 3D coordinates.

The problem of SBDD can be posed as modeling the conditional distribution $p(G^M | G^P)$. The challenge is to generate candidate ligands with high binding affinity to a target protein. We begin our analysis with a theoretical examination of expressivity limits of GNNs in the SBDD context.

\section{REPRESENTATION LIMITS OF GNNS FOR SBDD}\label{sec:th_repr}
The generative models for SBDD are quickly increasing in size (see Table \ref{tab:main_results}) aiming to expand their expressivity. However, we bring attention to the representational issues concerning large models. Specifically, we show that there exist molecular graphs that cannot be distinguished by GNN models regardless of their depth, even in the additional protein context. Consequently, certain graph properties, including the number of cycles, diameter, and sizes of cycles, cannot be learned by message-passing graph neural networks (MPGNNs). Interestingly, as we will demonstrate in Section \ref{sec:fastsbdd}, these features are very useful in predicting binding affinity.

It is known that there exist distinct molecular structures that cannot be distinguished using message-passing graph neural networks (MPGNNs) \citep{wl_mol, gnns_repr}.
Following the notation of \citet{gnns_repr}, we consider \textit{Locally Unordered} GNNs (LU-GNNs), summarized as follows. The updated embedding $h_v^{(l)}$ of node $v$ at layer $l$ is defined as
\begin{equation*}
\begin{split}
    m_{u \to v}^{(l-1)} & = \phi(h_u^{(l-1)}, e_{uv}) \\
    \Tilde{h}_v^{(l-1)} & = \mbox{AGG} \lbrace m_{u \to v}^{(l-1)} \: | \: u \in N(v) \rbrace \\
    h_v^{(l)} & = \mbox{COMBINE} \lbrace h_v^{(l-1)}, \Tilde{h}_v^{(l-1)} \rbrace,
\end{split}
\end{equation*}
where $N(v)$ denotes the set of neighbours of $v$, $e_{uv}$ are edge features and $\phi$ is any function. It is known that there exist non-isomorphic graphs, for which LU-GNNs produce identical embeddings \citep{gnns_repr}. We now note that this result also holds for graphs embedded in 3D space. Let us introduce \textit{Locally Unordered 3D} GNNs (LU3D-GNN) which aggregate embeddings for each node based on the embeddings of the neighbours, as well as their distance to these neighbours. Formally, in LU3D-GNNs, the messages are defined as:
\begin{equation}
    m_{u \to v}^{(l-1)} = \phi\left(h_u^{(l-1)}, e_{uv}, \lVert x_u - x_v \rVert\right)
\end{equation}
where $x_v$ is the 3D position of $v$. We have the following result.

\begin{restatable}[]{lemma}{singlebody}
\label{lemma:3d_single_body}
    There exist connected non-isomorphic geometric graphs that differ in the number of conjoined cycles, girth, size of the largest cycle and cut-edges that LU3D-GNNs cannot distinguish.
\end{restatable}

On the lefthand side of Figure \ref{fig:lu-gnns}, we show two graphs that are not isomorphic and differing in all mentioned properties, but which cannot be distinguished by LU3D-GNNs. For clarity of presentation, we presented the graphs in 2D, but any 3D configuration having all edges of equal length can be used.
We prove that they cannot be distinguished by LU3D-GNNs in Appendix \ref{app:repr_limits}.

We extend this result to the SBDD context. Specifically, we analyze pairs of graphs corresponding to protein-ligand complexes. Let $G_1=(V_1, E_1), G_2=(V_2, E_2)$ be any two graphs and $\mathcal{C}(G_1, G_2)$ denote a \textit{complex} graph, i.e. $\mathcal{C}(G_1, G_2)= (V, E)$, where
\begin{equation}
\label{eq:complex_graph}
    V = V_1 \cup V_2 \mbox{ and } E = E_1 \cup E_2 \cup V_1 \times V_2
\end{equation}
and, importantly, features for all added edged $e_{uv}  \in V_1 \times V_2$ only depend on the features of nodes at their endpoints. We have the following results.
\begin{restatable}[]{proposition}{multibody}
\label{prop:multi_body}
\leavevmode
\begin{enumerate}[label={(\roman*)}]
    \item If $G_1$ and $G_2$ are indistinguishable for LU-GNNs, then for any graph P,  the complex graphs $\mathcal{C}(P, G_1)$, $\mathcal{C}(P, G_2)$ are also indistinguishable for LU-GNNs. \label{prop:1}
    \item There exist ligand graphs $G_1, G_2$ differing in graph properties listed in Lemma \ref{lemma:3d_single_body}, such that for any protein $P$, the complexes $\mathcal{C}(P, G_1)$, $\mathcal{C}(P, G_2)$ are indistinguishable for models using LU3D-GNNs or LU-GNNs for intra-ligand and intra-protein message passing and LU-GNNs for inter ligand-protein message passing. \label{prop:2}
\end{enumerate}
\end{restatable}
The righthand side of Figure \ref{fig:lu-gnns} visualizes the proposed statement. The proof can be found in Appendix \ref{app:repr_limits}.
Proposition \ref{prop:multi_body} shows that the representational issues of GNNs carry over to the context of SBDD. Specifically, this shows that there exist distinct protein-ligand structures that cannot be distinguished with LU-GNN models regardless of their depth or number of parameters. Importantly, LU-GNNs' inability to distinguish graphs that differ in certain properties implies that LU-GNNs are unable to learn these properties.

\begin{figure*}
    \centering
    \includegraphics[width=0.95\textwidth]{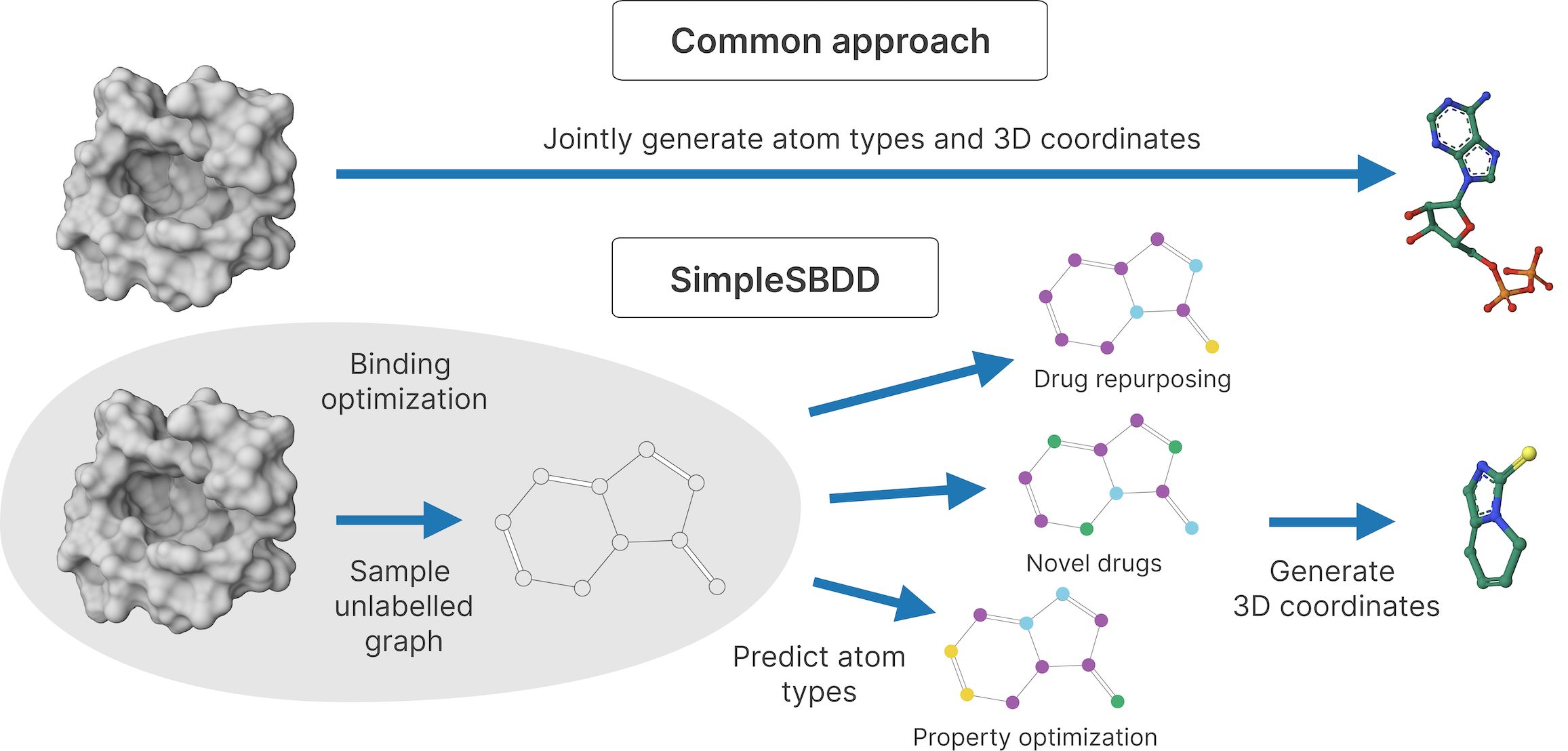}
    \caption{\textbf{Comparison of SimpleSBDD to common approaches}. Top: SBDD approaches commonly learn to approximate the data distribution of atom types and 3D coordinates conditioned on the protein pocket. Bottom: SimpleSBDD first generates the unlabelled graph explicitly optimized for estimated binding affinity. Then it predicts atom types using different strategies designed for solving different tasks independently of the protein pocket. Finally, it generates a 3D configuration.}
    \label{fig:model_flow}
\end{figure*}

We now take a different perspective and show that generative models for SBDD can be significantly reduced in size and computational complexity without sacrificing performance.

\section{SIMPLE STRUCTURE-BASED DRUG DESIGN}
\label{sec:fastsbdd}

In this section, we propose a significantly simplified generative model for SBDD with two crucial components: innovative model decomposition and explicit incorporation of estimated binding optimization within the model. This differs from the prevailing paradigm that primarily focuses on learning the data distribution.

\subsection{Decoupling the unlabelled molecular graph from atom types}\label{sec:decomp_just}

One of the main innovations underlying our framework is the separation of molecular representation into the unlabelled molecular graph and the labels, the atom types. 
Our motivation to adopt such an approach is primarily rooted in empirical findings that highlight the strong relationship between the structure of molecules and their predicted binding affinities, as approximated by the Vina software.
In a preliminary experiment, we investigated how Vina binding scores behave when we modify the atom types and coordinates while keeping the underlying graph intact.

Remarkably, our analysis showed a high correlation at $\rho=0.83$, between the scores of the original molecules and their modified counterparts.
This result highlights that a considerable portion of the variability in predicted binding affinity can be attributed to the information contained within the unlabelled molecular graph itself, even before incorporating atom-specific details (details in Appendix \ref{app:model_decomp}). Note that features characterizing the unlabeled molecular graph, such as the diameter and the number of rings, are unlearnable by LU-GNNs (Section \ref{sec:th_repr}). We will leverage this insight in our model.

It is important to recognize that these findings pertain to the Vina software and not the biological process of binding.
While Vina is the most prevalent approximation, it remains just that – an approximation.
Nonetheless, upon conducting similar experiments with Gnina \citep{gnina}, a reportedly significantly more accurate binding approximation than Vina, we found similar results (Appendix \ref{app:gnina}).
This suggests that the observed behavior is not solely an artifact of Vina but may be more widespread.

Inspired by these findings, we propose a new model SimpleSBDD (Figure \ref{fig:model_flow}). We first generate the high-binding unlabelled graph structures $U^M$, and generate atom types $\a^M$ independently of the protein. The model is defined to approximate the conditional distribution
\begin{equation}
    p(G^M | G^P) \approx p(\a^M, \s^M | U^M\cancel{, G^P}) p(U^M | G^P),
\end{equation}
where we assume atoms $\a^{M}$ are conditionally independent of the protein pocket $G^P$ given the unlabelled graph structure $U^M$. This approximation does incur information loss as expected. However, as we will see in the next section, performance-aware learning mitigates this effect.
The generative model thus reduces to two independent components: unlabelled graph sampler and atom sampler. We discuss them below.

\subsection{Performance-aware learning}
\label{sec:struct_sampler}

Typically, the unlabelled graph model $p(U^M | G^P)$ is trained to match the data distribution. However, we hypothesize that observed ligand-protein complexes contain examples with suboptimal binding, and propose to focus on the most promising complexes with the following procedure. We first learn a scoring model $g_\theta$ to predict binding affinity from unlabelled graph alone,
\begin{equation}\label{eq:scoring_model}
    g_{\theta}(U^M, G^P) \approx \mbox{Vina}(G^M, G^P).
\end{equation}
To represent $U^M$ we use properties describing its structure such as its number of rings, number of rotatable bonds and graph diameter, which cannot be learned by MPGNNs (See Appendix \ref{app:vina_approx} for details).
Next, we use $g_\theta$ to score unlabelled graphs from a repository of molecules, and choose those with the best predicted affinities. We define a generative distribution 
\begin{equation}
\label{eq:connectivity_sampler}
\begin{split}
    &p_{\theta}(U^M | G^P) = \mbox{UNIFORM}(U(G^P)), \\
    &U(G^P)  = \big\{ U^M \in \mathcal{U}^{M} \: | \: g_{\theta}(U^M, G^P) \in \left[v_{\mathrm{min}}, v_{\mathrm{max}} \right]  \big\},
\end{split}
\end{equation}
where $\mathcal{U}^{M}$ is a database of unlabelled graph structures, and $v_{\mathrm{min}}, v_{\mathrm{max}}$ are threshold hyperparameters. We set $v_{\mathrm{min}}$ and $v_{\mathrm{max}}$ to be the 5th and 10th percentiles of all predictions for $\mathcal{U}^{M}$ to avoid outliers (Appendix \ref{app:ba_qed_tradeoff}). As $\mathcal{U}^{M}$, we used the ZINC250k \citep{zinc} instead of the CrossDocked2020 \citep{crossdocked}, which only contains 8K unique ligands.

We note that instead of sampling from a database, one can define a generative model over unlabelled graphs.
We have not investigated this due to substantial molecular variability even with predefined unlabelled graphs.
Specifically, one can generate at least a thousand distinct and chemically diverse molecules sharing the same unlabelled graph (Appendix \ref{app:chemical_diversity}).

\paragraph{Training data for $g_{\theta}$}
To train the scoring model $g_{\theta}$, we need pairs of protein-ligand pairs with their affinity estimates. We chose not to use the CrossDocked2020 dataset, the gold standard for SBDD, for the following reasons: 1) even though it contains 100K protein-ligand pairs, there are only around 8K unique ligands, 
\begin{figure}
\begin{subfigure}{0.23\textwidth}
         \centering
         \includegraphics[width=\textwidth]{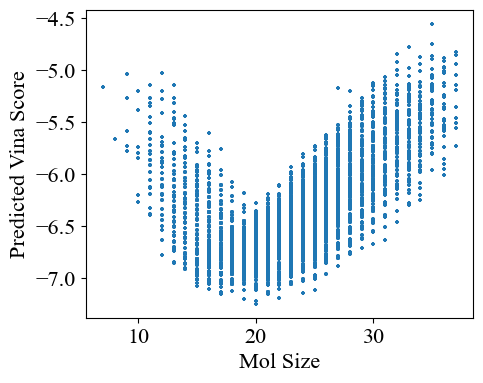}
         \caption{Optimal ligand size $\approx 20$}
         \label{fig:ba_vs_size_a}
     \end{subfigure}
     \hfill
     \begin{subfigure}{0.23\textwidth}
         \centering
         \includegraphics[width=\textwidth]{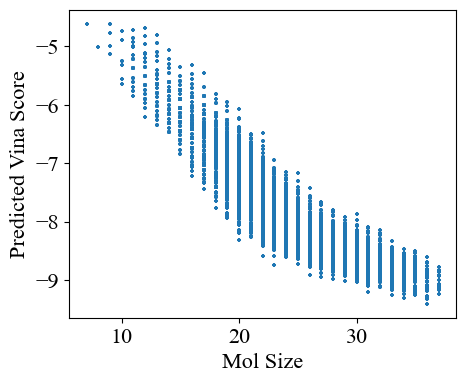}
         \caption{Optimal ligand size $\ge 38$}
         \label{fig:ba_vs_size_b}
     \end{subfigure}
    \caption{\textbf{Optimal ligand size has two modes}. Representative examples showing that the scoring model learns the optimal ligand size which is different for different proteins.}
    \label{fig:ba_vs_size}
\end{figure}
2) for a given protein pocket there may exist ligands with better affinity than what is present in the data and 3) for the scoring model to assign good values only to good structures, it needs to be trained also on examples with very poor binding affinity which are not available.

To address the above, we constructed a dataset ourselves. We sampled 1000 diverse ligand-protein complexes from CrossDocked2020, and included 50 additional random molecules from the ZINC250k. This resulted in 51,000 pairs, for which we computed the Vina scores. The scoring model $g_{\theta}$ was trained with stochastic gradient descent to minimize the mean squared error between prediction $g$ and Vina scores. See Appendix \ref{app:scoring_model_training} for details.

\paragraph{Scoring model detects optimal molecule size}
We visualize the predictions of $g$ as a function of the size of the molecule. We notice that protein pockets fall under two cases: either there is an optimal size of the ligand (Figure \ref{fig:ba_vs_size_a}) or there is a monotonic relationship favouring larger ligands (Figure \ref{fig:ba_vs_size_b}). This is because the largest available molecule is 38 atoms, which is smaller than the learned optimum.

\subsection{Atom sampler}
\label{seq:atom_sampler}
The atom sampler $p(\a^M, \s^M | U^M)$ generates atom types and coordinates of each atom. As mentioned in Section \ref{sec:decomp_just}, the precise 3D conformation has marginal effect on the Vina score (due to redocking), and therefore we simply generate any valid 3D configuration with RDKit. We choose a factorisation
\begin{equation}
    p(\s^M, \a^M | U^M) = p( \s^M | \a^M, U^M) p(\a^M | U^M),
\end{equation}
where $p( \s^M | \a^M, U^M)$ is an off-the-shelf conformation generation model available in RDKit, which we discuss in more detail in Appendix \ref{app:pose_var}. In practice, the 3D configuration of the ligand must be in the vicinity of the protein pocket in order to correctly compute its binding affinity using Vina. We therefore train a separate model which predicts the center of mass of the ligand based on the protein pocket, and use it to center the conformers. See Appendix \ref{app:com_predictor} for more details.

To define the generative model over the atom types given unlabelled structure, we adapt an existing generative model MoFlow \citep{moflow}, consisting of an unlabelled graph model $f_U$ and a conditional atom model $f_{\a | U}$.
The atom model induces a conditional distribution $p_{f_{\a | U}}(\a^M | U^M)$, which we use as our atom sampling distribution:
\begin{equation}\label{eq:atom_type_sampler}
    p(\a^M|U^M) \coloneqq p_{f_{\a | U}}(\a^M | U^M).
\end{equation}
We use a pre-trained MoFlow, so omit it from the trainable parameter count (Table \ref{tab:main_results}).

\begin{table*}[t]
   \small
    \caption{\textbf{SimpleSBDD finds high-quality drug candidates up to 1000x faster than competing methods}. Properties of CrossDocked2020 data and generation methods. $\dagger$ denotes method run exclusively on CPU. "-" denotes that a metric was not reported and we were unable to estimate. Error bars correspond to the standard deviation across test protein pockets.
    See Appendix \ref{app:baselines} for comparison details.}
    \label{tab:main_results}
    \centering
    \resizebox{0.95\linewidth}{!}{
    \begin{tabular}{L{3cm} C{1.8cm} C{1.7cm} c c c c C{1.7cm} c}
    \toprule
         & Vina Score (kcal/mol, ↓) & High Affinity (↑) &  QED (↑) & SA (↑) & Diversity (↑) & Novelty (↑) & \#Params (↓) & Time (s, ↓)  \\
    \midrule
         Test set & -6.99 \textcolor{gray}{± 2.16} &  - & 0.48 \textcolor{gray}{± 0.21} & 0.73 \textcolor{gray}{± 0.14} & - &- &- & - \\
         \cmidrule(lr){1-9}
        DiffSBDD (\citeyear{diffsbdd}) & -6.29 \textcolor{gray}{± 1.93} & 0.37 \textcolor{gray}{± 0.31} & 0.49 \textcolor{gray}{± 0.19} & 0.63 \textcolor{gray}{± 0.14} & \textbf{0.79} \textcolor{gray}{± 0.07} & \textbf{0.54}  \textcolor{gray}{± 0.14}&  3.5M & 135 \textcolor{gray}{± 52} \\
        Pocket2Mol (\citeyear{pocket2mol}) & -7.10 \textcolor{gray}{± 2.56} & 0.55 \textcolor{gray}{± 0.31} & 0.57 \textcolor{gray}{± 0.16} & 0.74 \textcolor{gray}{± 0.13}& 0.72 \textcolor{gray}{± 0.16} & 0.45  \textcolor{gray}{± 0.16}& 3.7M & 2504 \textcolor{gray}{± 220} \\
        FLAG (\citeyear{flag}) & -7.25 \textcolor{gray}{± 2.25} & 0.58 \textcolor{gray}{± 0.24} & 0.50 \textcolor{gray}{± 0.17} & 0.75 \textcolor{gray}{± 0.16}& 0.70 \textcolor{gray}{± 0.15} &0.44 \textcolor{gray}{± 0.17}& 11M & 1048 \textcolor{gray}{± 682} \\
        DrugGPS (\citeyear{drug-gps}) & -7.28 \textcolor{gray}{± 2.14} & 0.57 \textcolor{gray}{± 0.23} & \textbf{0.61} \textcolor{gray}{± 0.22} & 0.74 \textcolor{gray}{± 0.18}& 0.68 \textcolor{gray}{± 0.15} &0.47 \textcolor{gray}{± 0.15}& 14.7M & 1008 \textcolor{gray}{± 554} \\
        TargetDiff (\citeyear{target-diff}) & -6.91 \textcolor{gray}{± 2.25} & 0.52 \textcolor{gray}{± 0.32} & 0.48 \textcolor{gray}{± 0.20} & 0.58 \textcolor{gray}{± 0.13} & 0.72 \textcolor{gray}{± 0.09} &0.47  \textcolor{gray}{± 0.14}& 2.5M & 3428 \textcolor{gray}{± NA} \\
        DecompDiff (\citeyear{decompdiff}) & -6.76 \textcolor{gray}{± 1.64} & 0.46 \textcolor{gray}{± 0.36} & 0.45 \textcolor{gray}{± 0.21} & 0.61 \textcolor{gray}{± 0.14}& 0.68 \textcolor{gray}{± 0.10} &0.52  \textcolor{gray}{± 0.13}& 5.0M & 6189 \textcolor{gray}{± NA} \\
        D3FG (\citeyear{D3FG}) & -6.96 \textcolor{gray}{± NA} & 0.46 \textcolor{gray}{± NA} & 0.50 \textcolor{gray}{± NA} & \textbf{0.84} \textcolor{gray}{± NA}& - &-& - & - \\
        EQGAT-diff (\citeyear{EQGAT-diff}) & -7.42 \textcolor{gray}{± 2.33} & - & 0.52 \textcolor{gray}{± 0.18} & 0.70 \textcolor{gray}{± 0.20}& 0.74 \textcolor{gray}{± 0.07} &-& 12.3M & - \\
        \cmidrule(lr){1-9}
        SimpleSBDD (Ours) & \textbf{-7.78} \textcolor{gray}{± 1.47} & \textbf{0.71} \textcolor{gray}{± 0.34} & \textbf{0.61} \textcolor{gray}{± 0.18} & 0.69 \textcolor{gray}{± 0.09}& 0.68 \textcolor{gray}{± 0.06}	 &0.51  \textcolor{gray}{± 0.10}&\textbf{23K} & \textbf{3.9}$^\dagger$ \textcolor{gray}{± 0.9} \\
    \bottomrule
    \end{tabular}
    }
\end{table*}
\section{RESULTS}\label{sec:results}

\paragraph{Data}
We follow \citet{pocket2mol, diffsbdd} and use the CrossDocked2020 dataset \citep{crossdocked} for evaluating the models. We use the same train-test split, which separates protein pockets based on their sequence similarity computed with MMseqs2 \citep{mmseq} and contains 100,000 train and 100 test protein-ligand pairs.

\subsection{Assessing general properties of sampled molecules}\label{sec:gen-comparison}
We follow the evaluation procedure described by \citet{pocket2mol}. For each of 100 test protein pockets, we sample 100 molecules. As evaluation metrics we report 1) \textbf{Vina Score} estimating binding affinity \citep{vina}, 2) \textbf{High Affinity}, which is the percentage of generated molecules with binding affinity at least as good as ground truth, 3) \textbf{QED}, a quantitative estimate of druglikeness, 4) \textbf{SA}, synthetic accesibility, 5) \textbf{Diversity} defined by average pairwise Tanimoto dissimilarity for the generated molecules in the pocket, 6) \textbf{Novelty}, defined by average Tanimoto dissimilarity to the most similar molecule from training set, 7) \textbf{\#{Params}}, the number of trainable parameters and 8) \textbf{Time} to generate 100 molecules in seconds. We include more molecular metrics in Appendix \ref{app:additional_metrics}. As baselines, we use eight recent methods for SBDD: Pocket2Mol \citep{pocket2mol}, DiffSBDD \citep{diffsbdd}, TargetDiff \citep{target-diff}, DecompDiff \citep{decompdiff}, FLAG \citep{flag}, DrugGPS \citep{drug-gps}, D3FG \citep{D3FG} and EQGAT-diff \citep{EQGAT-diff}.

We report the results in Table \ref{tab:main_results}. Clearly, SimpleSBDD significantly outperforms all baselines in terms of estimated binding affinity while being up to 1000x faster despite being run solely on CPU (all baselines use GPU for sampling). Additionally, our method has only 23k trainable parameters, which is two orders of magnitude less than the baselines. We provide more details on the baseline comparison in Appendix \ref{app:baselines}. We also provide a visualization of our model predictions in Figure \ref{fig:qualitative}.

\subsection{Additional evaluation}\label{sec:add_eval}
We have demonstrated that SimpleSBDD is very competitive on the standard evaluation criteria for SBDD models. We now discuss its broader applications.

\paragraph{Property optimization} Due to our novel decomposition of the model into the unlabelled graph and atom labels, we can optimize the estimated binding affinity and other properties simultaneously. This is especially important, since molecules without desirable drug-like properties are of diminished practical utility even if the predicted binding is strong. Specifically, given the target protein we first generate an unlabelled graph with strong predicted binding affinity. As we noted in Section \ref{sec:struct_sampler} and elaborated on in Appendix \ref{app:chemical_diversity}, there is large diversity in the chemical properties of molecules that share the unlabelled graph. Therefore, for the sampled unlabelled graph, there is enough flexibility in atom types to enable molecular property optimization.

As an illustrative example, we follow \citet{gomez2018automatic} and choose $5 \cdot \mbox{QED} + \mbox{SA}$ as the property mixture to optimize alongside predicted binding affinity. Concretely, 50 proposals are sampled using the atom type model defined in Equation (\ref{eq:atom_type_sampler}) and the candidate with the best property values is selected. We denote this model variant SimpleSBDD--$\mathcal{PO}$. With this additional procedure we can significantly improve all molecular metrics without sacrificing binding affinity whilst remaining an order of magnitude faster than baselines. See Appendix \ref{app:sbdd-po} for more details.

\begin{figure*}[t]
 \centering    
    \includegraphics[width=0.93\textwidth]{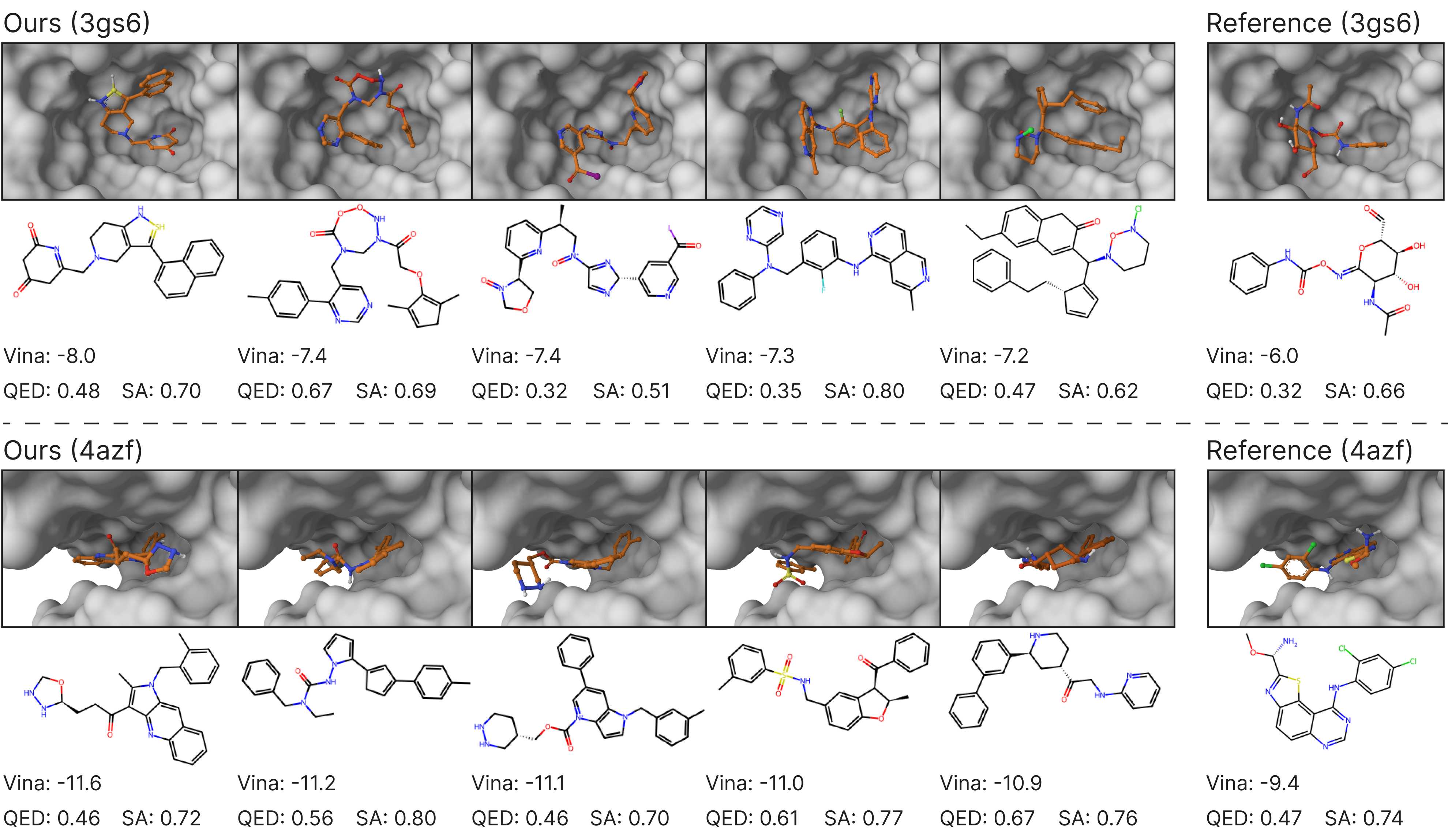}
        \caption{\textbf{SimpleSBDD generates diverse drug candidates with stronger predicted binding than reference molecules.} We visualize the predictions of our model for two randomly chosen proteins (PDB ids 3gs6 and 4azf). We choose 5 molecules with best predicted binding affinities for each protein and compare their Vina score, QED and synthetic accessibility with the reference molecule.}
        \label{fig:qualitative}
\end{figure*}

\paragraph{Comparison with optimization-based methods} So far we have compared our method with multiple baselines which generate molecules conditioned on a target protein. Neither SimpleSBDD nor any of these baselines has access to the Vina software during generation. We now also compare with \textit{optimization-based} methods, i.e. algorithms, which generate molecules by explicitly using Vina during sampling to directly optimize it. Specifically, we compare SimpleSBDD--$\mathcal{PO}$ with the three best performing methods, i.e. Reinforced Genetic Algorithm \citep{rga}, 3D-Monte Carlo Tree Search (3D-MCTS) \citep{3d-mcts} and AutoGrow4.0 \citep{autogrow4}. We present the results in Table \ref{tab:optimization_based_results}.

It is evident that SimpleSBDD--$\mathcal{PO}$ is significantly more efficient than all the other methods. Note that these baselines are not only the best performing, but also the fastest according to \citep{rga}. What makes them much slower than SimpleSBDD are the repeated evaluations of the Vina software, which is approximately 20000x slower than e.g. computing QED (20s vs 0.001s). Furthermore, the only method outperforming SimpleSBDD in terms of Vina score, AutoGrow4, is significantly worse in terms of QED. SimpleSBDD offers an attractive tradeoff between predicted binding affinity and the predicted quality of the generated drug candidates at a much lower computational cost.

\begin{table*}[t]
   \small
    \caption{\textbf{SimpleSBDD is significantly faster than optimization-based methods}. Comparison with optimization-based methods. $\dagger$ denotes method run exclusively on CPU. Error bars correspond to the standard deviation across test protein pockets.}
    \label{tab:optimization_based_results}
    \centering
    \resizebox{0.95\linewidth}{!}{
    \begin{tabular}{L{3cm} C{1.8cm} C{1.7cm} c c c C{1.7cm} c}
    \toprule
         & Vina Score (kcal/mol, ↓) & High Affinity (↑) &  QED (↑) & SA (↑) & Diversity (↑) & \#Params (↓) & Time (s, ↓)  \\
    \midrule
         Test set & -6.99 \textcolor{gray}{± 2.16} &  - & 0.48 \textcolor{gray}{± 0.21} & 0.73 \textcolor{gray}{± 0.14} & - &- & - \\
         \cmidrule(lr){1-8}
        RGA (\citeyear{rga}) & -6.93 \textcolor{gray}{± 1.17} & 0.53 \textcolor{gray}{± 0.41} & 0.46 \textcolor{gray}{± 0.15} & \textbf{0.80} \textcolor{gray}{± 0.07} & \textbf{0.76} \textcolor{gray}{± 0.01} &  341K & 11576 \textcolor{gray}{± 3717} \\
        3D-MCTS (\citeyear{3d-mcts}) & -7.55 \textcolor{gray}{± 1.32} & 0.66 \textcolor{gray}{± 0.38} & 0.65 \textcolor{gray}{± 0.14} & 0.78 \textcolor{gray}{± 0.07} & 0.62 \textcolor{gray}{± 0.07} &  \textbf{0} & 4150 \textcolor{gray}{± 313} \\
        AutoGrow4 (\citeyear{autogrow4}) & \textbf{-8.33} \textcolor{gray}{± 1.55} & \textbf{0.81} \textcolor{gray}{± 0.28} & 0.36 \textcolor{gray}{± 0.17} & 0.67 \textcolor{gray}{± 0.10} & 0.65 \textcolor{gray}{± 0.06} &  \textbf{0} & 10800$^\dagger$ \textcolor{gray}{± 0} \\
        \cmidrule(lr){1-8}
        SimpleSBDD--$\mathcal{PO}$ & -7.98 \textcolor{gray}{± 1.46}  & 0.75 \textcolor{gray}{± 0.35} & \textbf{0.80} \textcolor{gray}{± 0.10} & 0.73 \textcolor{gray}{± 0.08}& 0.67 \textcolor{gray}{± 0.06}	 &23K & \textbf{115}$^\dagger$ \textcolor{gray}{± 11} \\
    \bottomrule
    \end{tabular}
    }
\end{table*}

\paragraph{Efficient scoring for drug repurposing} SBDD is a problem of generating novel drugs likely to bind to a given target pocket.
However, an equally important application is \textit{drug repurposing}, a problem of investigating new uses for existing molecules, which is closely related to determining whether a drug can effectively bind to a specific protein target.
This is a significant application since, unlike samples generated by machine learning models, existing drugs have experimentally verified properties and their synthesizability is established.

We can use our scoring model (Equation (\ref{eq:scoring_model})) to scan databases of existing drugs and select ones with the best predicted binding.
Moreover, our implementation of the scoring model allows us to do that very efficiently due to the fact that protein embedding needs only be computed once when scoring multiple molecules.
Specifically, it allows for \textit{scoring 9100 molecules per second on a single CPU}.

To showcase the practicality and efficiency of our scoring model in the context of drug repurposing, we perform the following experiment.
For each protein pocket, we randomly sample 16384 molecules from a database and select 100 of them with predicted Vina scores between 5th and 10th percentile.
We then compare the average binding affinity and other molecular properties with the reference molecules of the target protein.
We found that for each protein pocket it takes around 2 seconds to find 100 diverse molecules with much higher binding affinity and improved other molecular metrics.
Furthermore, we repeated the experiment with ChEMBL dataset \citep{chembl1, chembl2} and found that $g_\theta$ generalizes well across molecular databases.
See Appendix \ref{app:drug_repurposing} for details.

\section{CONCLUSION AND BROADER IMPACT}\label{sec:conclusion}
We advocate a shift in focus from mere model expressivity to robust generalization.
Specifically, we bring to the forefront inherent limitations in expressivity of GNNs and prove that they carry over to the multibody domain of SBDD.
Furthermore, we show that with a novel model decomposition and explicit binding optimization built into the model, we can outperform state of the art with significantly simplified models.
The unprecedented efficiency and success of SimpleSBDD highlights the potential to reshape the approach to SBDD but also paves the way for streamlining docking software, especially when assessed with widely adopted tools like Vina and Gnina.

However, there are limitations to our findings. In particular, while we rely on docking software such as Vina, the gold standard for evaluation in the SBDD realm, the ideal validation would entail wet lab experiments or molecular dynamics simulations.
That said, such methods are often prohibitively expensive and could utilize the significant resources afforded by models such as SimpleSBDD.

The very nature of SimpleSBDD as a generative model, especially with its property optimization capabilities, brings with it risks.
As SBDD models improve at designing molecules with specific properties, we must remain vigilant of the potential unintended outcomes, since streamlining the design process could accelerate the design of harmful biochemicals.

\section{LIMITATIONS}
We acknowledge limitations of current docking-based approaches in real-world drug discovery. While docking scores are a widely used metric in SBDD, they should be interpreted with caution due to the following concerns:

\textbf{Docking Score Reliability:} A high docking score does not necessarily guarantee success in subsequent molecular simulations or experimental validations. Docking scores indicate the predicted binding pose of a ligand to a protein target and should not be overinterpreted as definitive indicators of biological activity.

\textbf{Dataset Curation and Bias:} The quality and curation of datasets can significantly impact docking outcomes. Protein structures from the Protein Data Bank (PDB) often exist as complexes with specific ligands. It is not uncommon to delete the ligand from the complex and dock new molecules instead. This can influence the docking results.

\textbf{Biological Context:} Docking scores do not distinguish between different functional outcomes of binding, such as whether a molecule acts as an inhibitor or an agonist. This limits their utility in predicting the nuanced biological effects of candidate molecules

\textbf{Correlation with Experimental Data:} Docking scores and experimental binding affinities are not always strongly correlated, and docking methods may lack proper calibration against real-world biochemical and biophysical data.

\textbf{Binding Site Challenges:} The binding site of a protein target might be suboptimal or unsuitable for a candidate ligand, further complicating the interpretation of docking results.

Such limitations underscore the need to treat docking scores with caution, and provide motivation for our work – we urge the community to critically evaluate the existing paradigm and redirect the efforts for SBDD (saving the computation expense and time using accelerated approaches such as the one we advocate, and paying additional attention to more reliable investigations such as molecular simulations).

\subsection*{Acknowledgments}
This work was supported by the Finnish Center for Artificial Intelligence (FCAI) under Flagship R5 (award 15011052). VG also acknowledges the support from Saab-WASP (grant 411025), Academy of Finland (grant 342077), and the Jane and Aatos Erkko Foundation (grant 7001703). RK thanks Paulina Karczewska for her help with preparing figures.

\bibliography{references}

\section*{Checklist}



 \begin{enumerate}

 \item For all models and algorithms presented, check if you include:
 \begin{enumerate}
   \item A clear description of the mathematical setting, assumptions, algorithm, and/or model. [Yes] Section \ref{sec:fastsbdd}
   \item An analysis of the properties and complexity (time, space, sample size) of any algorithm. [Yes] Table \ref{tab:main_results}
   \item (Optional) Anonymized source code, with specification of all dependencies, including external libraries. [Yes] We include code to reproduce our results as part of the submission.
 \end{enumerate}

 \item For any theoretical claim, check if you include:
 \begin{enumerate}
   \item Statements of the full set of assumptions of all theoretical results. [Yes] Proposition \ref{prop:multi_body} and Lemma \ref{lemma:3d_single_body}
   \item Complete proofs of all theoretical results. [Yes] Appendix \ref{app:repr_limits}
   \item Clear explanations of any assumptions. [Yes] Section \ref{sec:th_repr}     
 \end{enumerate}

 \item For all figures and tables that present empirical results, check if you include:
 \begin{enumerate}
   \item The code, data, and instructions needed to reproduce the main experimental results (either in the supplemental material or as a URL). [Yes] The evaluation procedure is described in Section \ref{sec:results} with more details in Appendix \ref{app:baselines} and the submitted code in the supplementary can be used to reproduce our results.
   \item All the training details (e.g., data splits, hyperparameters, how they were chosen). [Yes] Data is described in Section \ref{sec:results} and model details are provide in Appendices \ref{app:vina_approx} and \ref{app:com_predictor}
         \item A clear definition of the specific measure or statistics and error bars (e.g., with respect to the random seed after running experiments multiple times). [Yes] Tables \ref{tab:main_results}, \ref{tab:optimization_based_results}, \ref{tab:additional_metrics}, \ref{tab:po_results} and \ref{tab:repurposing}
         \item A description of the computing infrastructure used. (e.g., type of GPUs, internal cluster, or cloud provider). [Yes] Appendix \ref{app:scoring_model_training} and \ref{app:com_predictor}
 \end{enumerate}

 \item If you are using existing assets (e.g., code, data, models) or curating/releasing new assets, check if you include:
 \begin{enumerate}
   \item Citations of the creator If your work uses existing assets. [Yes] The datasets CrossDocked2020, ZINC250k and ChEMBL datasets are properly cited. The MoFlow model that we used is also properly cited.
   \item The license information of the assets, if applicable. [Yes] Appendix \ref{app:licences}
   \item New assets either in the supplemental material or as a URL, if applicable. [Yes] A newly curated dataset described in Section \ref{sec:fastsbdd} is included in the submitted supplementary material.
   \item Information about consent from data providers/curators. [Not Applicable]
   \item Discussion of sensible content if applicable, e.g., personally identifiable information or offensive content. [Not Applicable]
 \end{enumerate}

 \item If you used crowdsourcing or conducted research with human subjects, check if you include:
 \begin{enumerate}
   \item The full text of instructions given to participants and screenshots. [Not Applicable]
   \item Descriptions of potential participant risks, with links to Institutional Review Board (IRB) approvals if applicable. [Not Applicable]
   \item The estimated hourly wage paid to participants and the total amount spent on participant compensation. [Not Applicable]
 \end{enumerate}

 \end{enumerate}

\newpage
\appendix
\onecolumn
\section{DECOUPLING THE UNLABELLED MOLECULAR GRAPH FROM THE ATOM TYPES}\label{app:model_decomp}
Here we provide details of experiments we discuss in section \ref{sec:decomp_just}.
\subsection{Vina score}
Vina uses a scoring function, which approximates the experimentally measured binding affinity with several types of interactions between the protein and ligand atoms \citep{scoring_function}:
\begin{equation}
\label{eq:scoring_function}
    f(G^{P}, G^{M}) = \sum_{k=1}^K w_k
 \sum_{i=1}^{N_P} \sum_{j=1}^{N_M}  f_k\Big( a_i^{P}, a_j^{M}, \s_i^{P}, \s_j^{M} \Big),
\end{equation}
where $\{f_k\}_{k=1}^K$ represent $K$ types of interatomic interactions that depend on atom types $a$ and 3D coordinates $\s$. The weights $\{w_k\}_{k=1}^K$ are pretrained to reflect experimentally measured affinities.
We emphasize that \textit{lower values of $f$ are interpreted as stronger binding.}
To estimate the binding affinity for a given protein-ligand complex $(G^{P}, G^{M})$, Vina performs a local search for the lowest-score 3D configuration of the molecule (a.k.a. \textit{redocking}). Specifically, gradient optimization is performed over global translations, global rotations, and torsional rotations $\mathcal{T}(G^M)$ of the ligand $G^M$ to optimise its fit to the protein pocket:
\begin{equation}\label{eq:binding_affinity}
    \mbox{Vina}(G^P, G^M) = \min \Big\{ f(G^P, \hat{G}^M) : \hat{G}^M \in \mathcal{T}(G^M) \Big\}.
\end{equation}
Vina uses the BFGS optimiser \citep{BFGS}.
Due to the \textit{redocking} procedure (\ref{eq:binding_affinity}), the Vina score can be interpreted as "binding potential", i.e. molecule’s binding affinity in its optimal pose.
\subsection{Influence of unlabelled molecular graphs on Vina scores}\label{sec:vina_experiments}
Equation (\ref{eq:binding_affinity}) indicates that if Vina employed an ideal optimizer to compute the global minimum, the initial 3D configuration of the molecule would not influence the outcome. Yet, because Vina utilizes an approximate gradient-based optimization method, variations in the initial 3D configuration can indeed affect the Vina score.
To measure that, we randomly selected 1,000 protein-ligand complexes from the CrossDocked2020 dataset \citep{crossdocked} and computed the Vina score both before and after modifying the ligand.
Here, by ``modifying the ligand", we refer to replacing its 3D configuration with a randomly generated conformer using RDKit. As expected, a high correlation at $\rho=0.97$ emerged indicating an insignificant impact of the ligand's initial 3D configuration on the Vina score (See Figure \ref{fig:conformation_invariance1}).

Furthermore, we experimented with changing ligand's atom types.
For a fixed unlabelled graph $U^M$, we re-sample random atom assignments $\a^M$ using MoFlow (Equation \ref{eq:atom_type_sampler}).
This procedure changes the molecule and therefore its 3D configuration is most likely no longer energetically valid.
We therefore use RDKit to generate a random conformer of the new molecule and use it to replace the original 3D coordinates.
Surprisingly, the correlation remains relatively high at $\rho=0.83$ (see Figure \ref{fig:atom_type_invariance1}).
These findings suggest that using just the unlabelled graph of the ligand, one can estimate the Vina score with considerable accuracy.
We use this observation in the definition of the model.
\begin{figure}[t]
    \centering
    \begin{subfigure}[b]{0.48\textwidth}
         \centering
         \includegraphics[width=\textwidth]{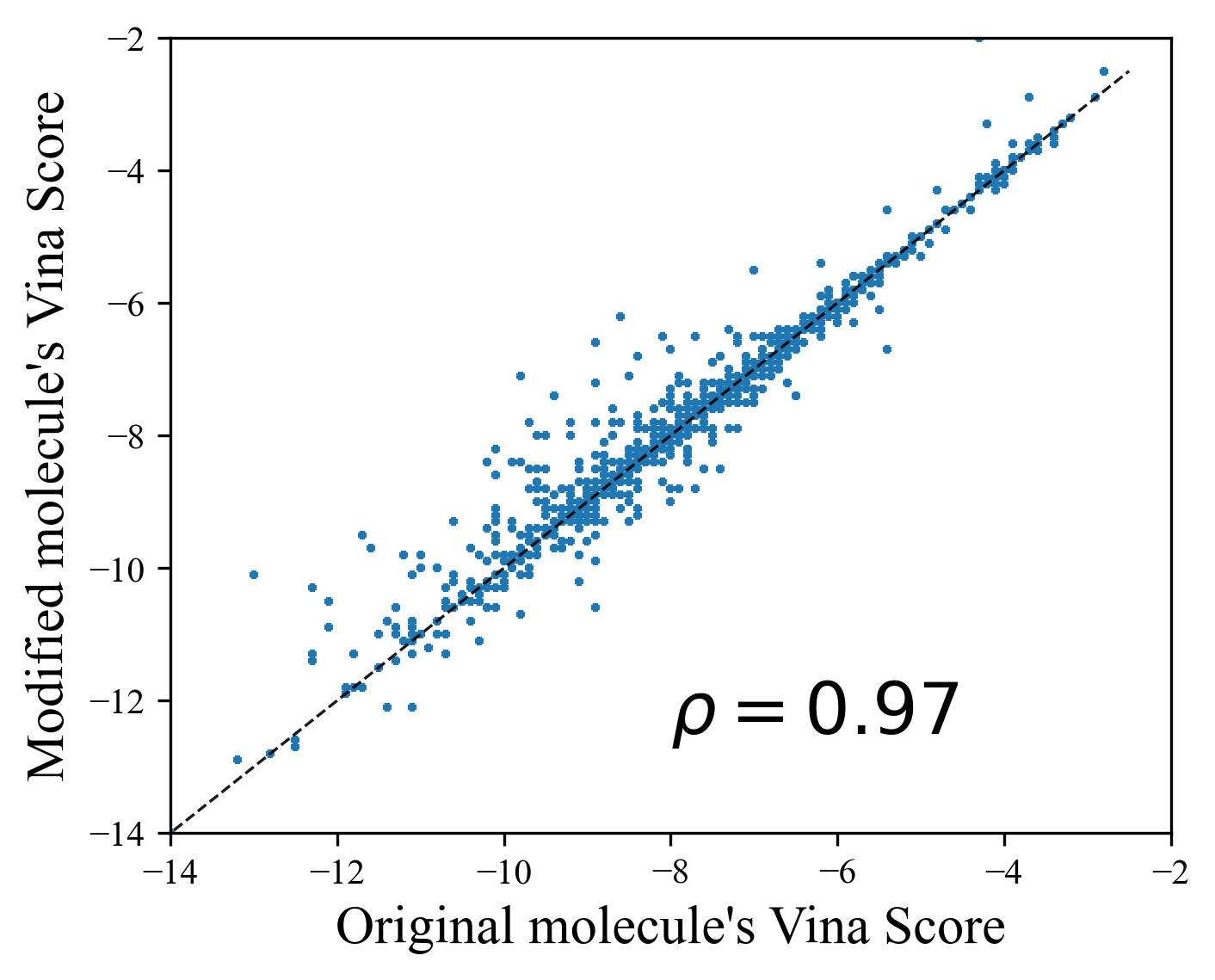}
         \caption{Changing initial 3D configuration.}
         \label{fig:conformation_invariance1}
    \end{subfigure}
    \hfill
    \begin{subfigure}[b]{0.48\textwidth}
         \centering
         \includegraphics[width=\textwidth]{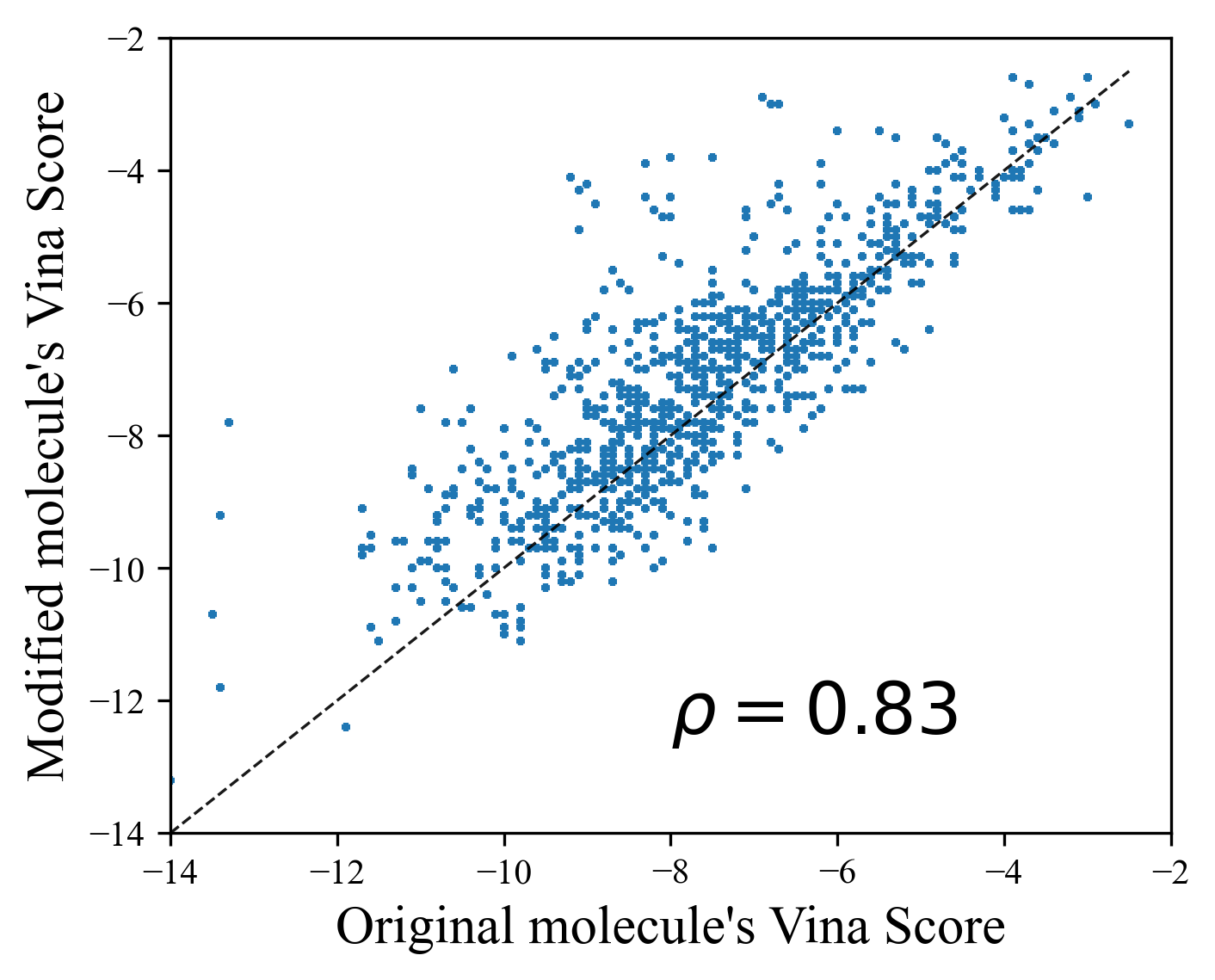}
         \caption{Changing atom types and initial 3D configuration}
         \label{fig:atom_type_invariance1}
    \end{subfigure}
    \caption{Impact of different transformations of the ligand on the Vina score.}
    \label{fig:binding_factors}
\end{figure}
\subsection{Justification of the model decomposition - Gnina}\label{app:gnina}
\begin{figure}[t]
    \centering
    
    \begin{minipage}{.10\textwidth}
    \textbf{Vina}
\end{minipage}%
\hfill
    \begin{minipage}{.42\textwidth}
        \includegraphics[width=\linewidth]{figures/conformation_invariance_corr.png}
    \end{minipage}%
    \hfill
    \begin{minipage}{.42\textwidth}
        \includegraphics[width=\linewidth]{figures/atom_type_invariance_corr_moflow.png}
    \end{minipage}%

        \begin{minipage}{.10\textwidth}
    \textbf{Gnina}
\end{minipage}%
\hfill
    \begin{minipage}{.42\textwidth}
        \includegraphics[width=\linewidth]{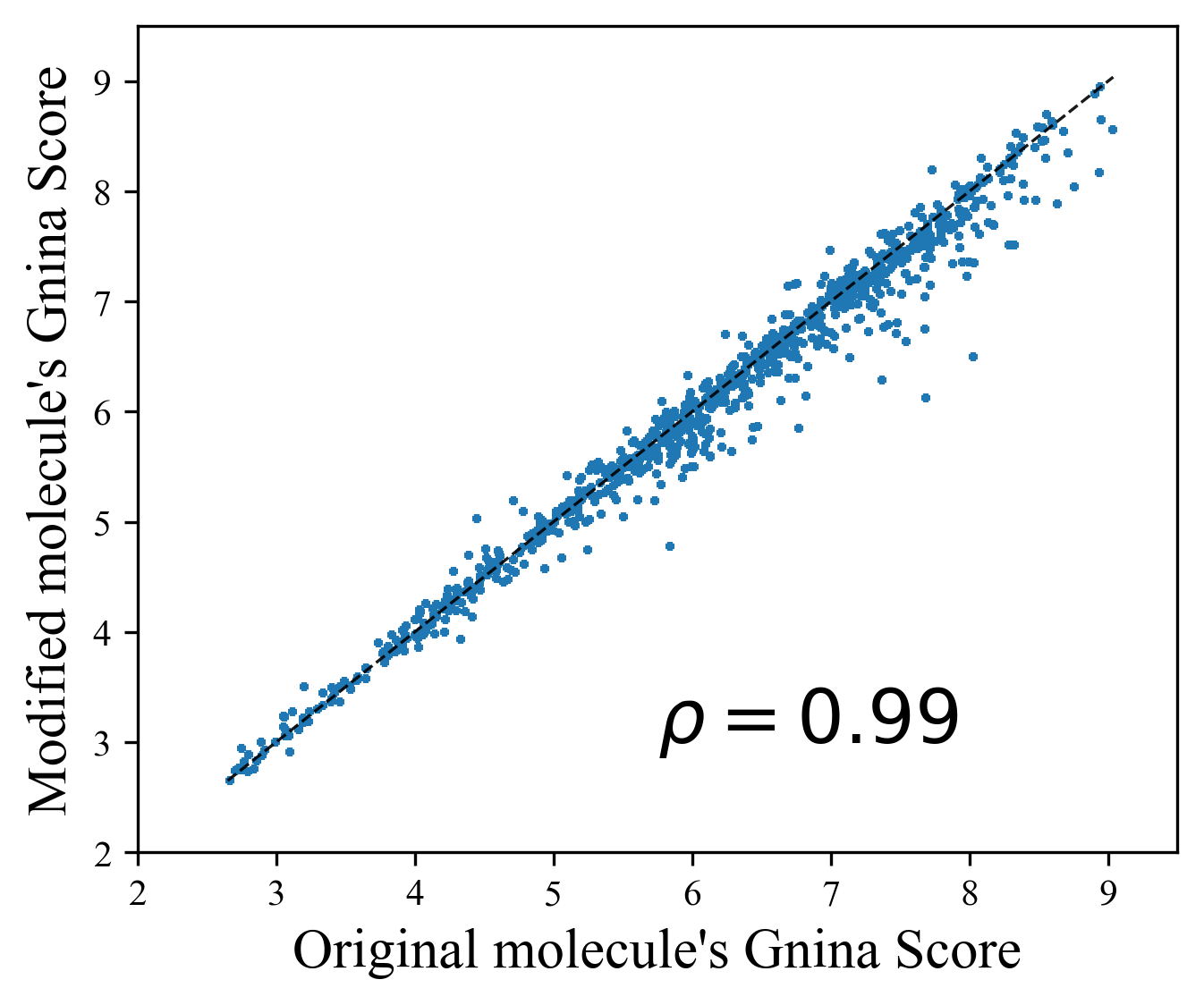}
        \subcaption{Changing initial 3D configuration}
    \end{minipage}%
    \hfill
    \begin{minipage}{.42\textwidth}
        \includegraphics[width=\linewidth]{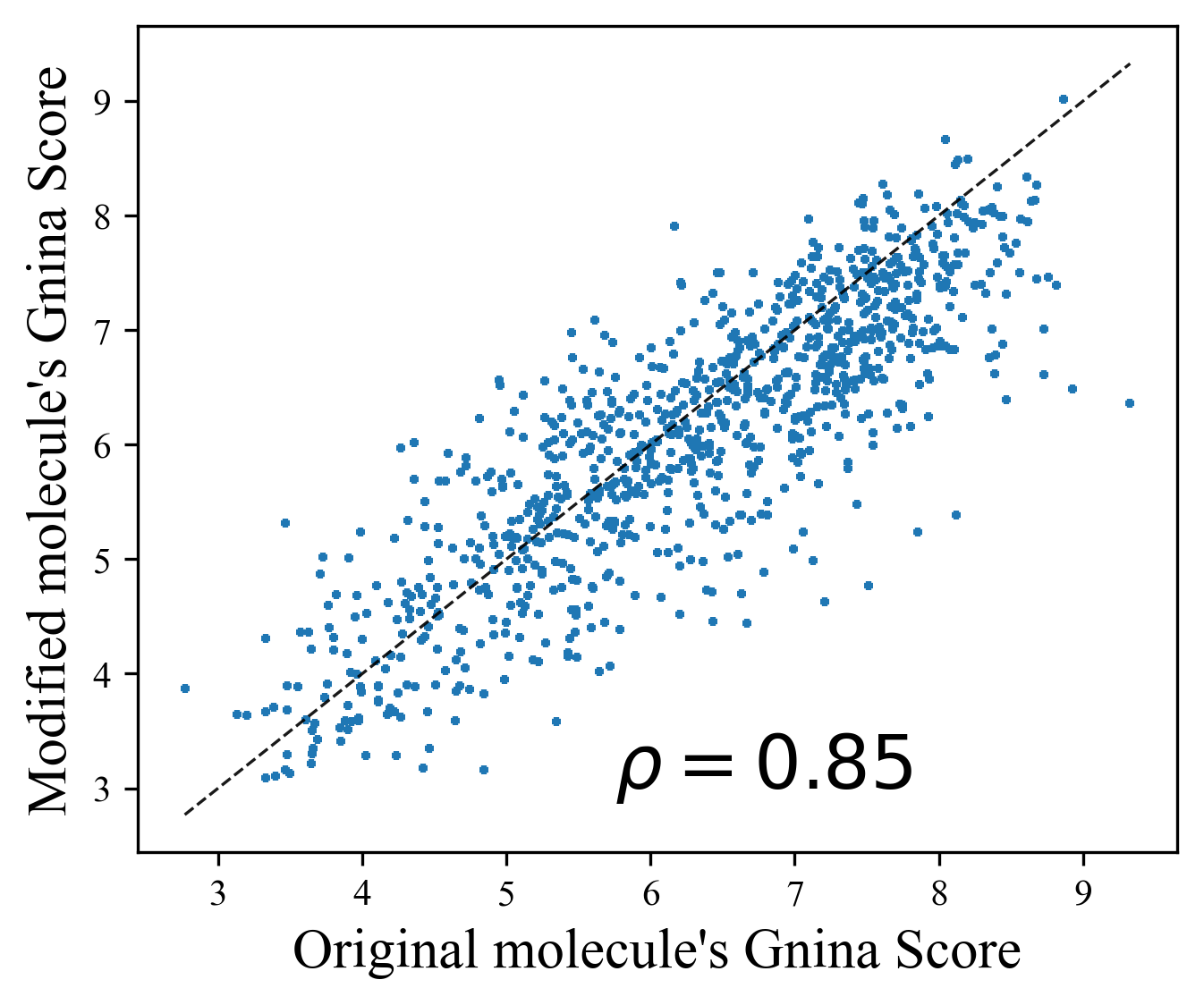}
        \subcaption{Changing atom types and initial 3D configuration}
    \end{minipage}%

    
    \caption{The impact of molecule modifications on binding scores is consistent across Vina and Gnina software.}
    \label{fig:global}
\end{figure}
In section \ref{sec:decomp_just}, we justified our model decomposition choice through empirical evidence indicating that the unlabelled molecular graph inherently contains substantial information, as reflected in the Vina score, even when the ligand's initial 3D configuration and atom type assignment are varied.
We now show that a similar phenomenon can be observed when, instead of Vina, a reportedly significantly more accurate binding software Gnina \citep{gnina} is used.

We repeat the experiments from the section \ref{sec:vina_experiments} with Gnina instead of Vina and find that for changing the initial 3D configuration, there is even smaller impact indicated by $\rho=0.99$ (due to Gnina also performing redocking). For changing atom type assignments, $\rho=0.85$ slightly higher than it was for Vina.
Please see Figure \ref{fig:global} for details.
\subsection{Generality of the findings}
We have demonstrated that significant portion of binding information, as measured by binding software, is contained in the unlabelled molecular graph.
So far, we performed all experiments on the gold-standard dataset for SBBD - CrossDocked2020.
We now repeat the analysis from section 4.1 for PDBind \citep{liu2015pdb}, specifically its \emph{refined} subset containing $\sim 5000$ highest-quality protein-ligand complexes.
Arguably, this is a better benchmark than CrossDocked2020, which has been augmented with protein-ligand complexes, without experimentally measured binding (the "cross-docking" procedure).

We performed the same experiment as in Section 4.1, i.e. measuring the impact of 1) changing the initial 3D configuration, and 2) changing atom types.
In both cases, we measured the binding with both Vina and Gnina.
We observed that the correlation is slightly lower, but still large enough to infer that a significant portion of the binding information is contained in the unlabelled graph.
This suggests that our findings hold more broadly across different datasets, including those compiled from experimental measurements. Please see Table \ref{tab:pdbind-results} for more details.

\begin{table}[h]
    \caption{Unlabelled molecular graph contains a significant portion of binding information. Findings hold across different binding software and datasets.}
    \label{tab:pdbind-results}
    \centering
    \resizebox{0.7\linewidth}{!}{
    \begin{tabular}{L{3.1cm} c c c}
    \toprule
       Experiment  &  Software & CrossDocked2020 & PDBind \\
    \midrule
    Changing atom & Vina & 0.827 & 0.781 \\
    types & Gnina & 0.849 & 0.797 \\
    \midrule
    Changing intitial 3D & Vina & 0.971 & 0.953 \\
    configuration & Gnina &  	0.989 	& 0.979 \\
    \bottomrule
    \end{tabular}
    }
\end{table}

\section{TRADEOFF BETWEEN BINDING AFFINITY PREDICTION AND QED}\label{app:ba_qed_tradeoff}
\begin{figure}
    \centering
    \includegraphics[width=0.6\textwidth]{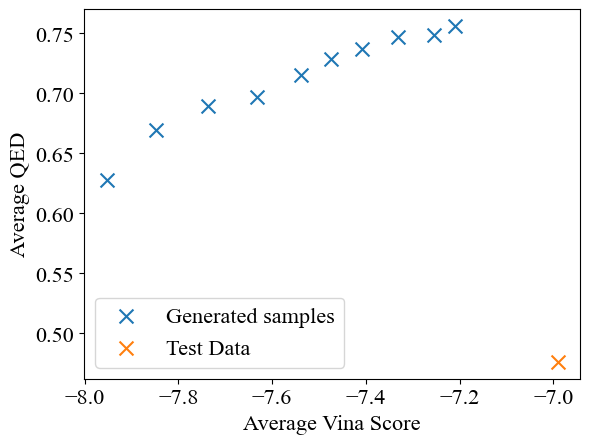}
    \caption{\textbf{Binding Affinity vs QED trade-off}. Our model allows for controlling the predicted binding affinity of generated molecules.}
    \label{fig:ba_vs_qed}
\end{figure}

As seen in Equation (\ref{eq:connectivity_sampler}), our method allows for controlling the binding affinity of the generated ligands. We therefore investigated the performance of the model when we vary the $v_\mathrm{min}, v_\mathrm{max}$ thresholds. Specifically, we evaluated 10 different versions of the unlabelled graph sampler with $(v_\mathrm{min}, v_\mathrm{max}) \text{ set to }  (q_0, q_5), (q_5, q_{10}), \ldots, (q_{45}, q_{50})$, where $q_k$ is the $k^{th}$ percentile of the predictions for $\mathcal{U}^M$ (note that these percentiles depend on the protein pocket). We note a trade-off between QED and Binding Affinity of these models, but all of them are still better on both of these criteria than the samples from the CrossDocked2020 dataset. We show this in Figure \ref{fig:ba_vs_qed}. We set $v_\mathrm{min}=q_5$ and $v_\mathrm{max}=q_{10}$ in our experiments.
\section{CHEMICAL DIVERSITY WITH FIXED UNLABELLED GRAPHS}
\label{app:chemical_diversity}
In section \ref{sec:struct_sampler}, we introduce the unlabelled graph sampler, which uses a set of existing molecular graphs to define possible 2D unlabelled graph structures that can be sampled, i.e. unseen graph structures cannot be sampled.
We now argue that this restriction still leaves substantial flexibility in the chemical space.

We performed the following experiment. We randomly sampled 100 molecules from the ZINC250k dataset. For each molecule, we used its template unlabelled graph to generate novel valid molecules by randomly replacing atom types without violating valency constraints. Whenever we generated a duplicate molecule we tried again. We terminated the procedure when we generated 1000 unique molecules or reached the limit of 10 failed trials due to duplicate generation.

For 98 of these 100 molecules we successfully generated 1000 unique novel molecules sharing the identical 2D unlabelled graph including bond types. For the remaining 2, the procedure terminated after generating 665 and 566 unique novel molecules respectively. This shows that in 98\% of the cases we are able to create at least 1000 novel valid molecules while keeping the unlabelled graph fixed.

Furthermore, we computed the average Tanimoto similarity between the original molecule and the newly generated ones sharing the unlabelled graph. We found that the average Tanimoto similarity was 0.28 ± 0.08 across the 100 sampled molecules from ZINC250k. We compared this number with the average pairwise Tanimoto similarity of the 100 molecules original molecules, which was 0.24 (within one standard deviation). We also contrast this number with the 0.85 similarity threshold, above which molecules have a high probability of having the same activity \citep{tanimoto96sim}.

This shows that, even for a fixed 2D molecular graph with defined bond types, we can define a large number of unique and valid molecules whose similarity to the original molecule is comparable with the pairwise molecular similarity in the ZINC250k dataset and thus proving the substantial chemical variability.
\section{BINDING AFFINITY APPROXIMATION MODEL DETAILS}\label{app:vina_approx}
In this section we provide more details of the scoring model from Eq. (\ref{eq:scoring_model}).
Recall that we want to approximate the binding affinity estimation defined in Eq. (\ref{eq:binding_affinity}) using the protein pocket information and ligand's unlabelled graph structure.
To that end we represent the protein pocket as a graph with residues as nodes.
Each node is equipped with a one-hot encoding of a residue type (1 of 20 possible amino acids) as a feature vector and 3D coordinates computed as a center of mass of the corresponding residue.
Two nodes are connected with an edge if they are at most 15Å apart and we limit each node to have at most 24 neighbours.
We do not use any edge features.
To compute the protein pocket embedding, we use E(3)-Equivariant Graph Neural Networks (EGNNs) \citep{egnn}.

An EGNN layer takes as input a set of node embeddings $\{\mathbf{h}^{(l-1)}_v\}_{v \in V}$, coordinate embeddings $\{\mathbf{x}^{(l-1)}_v\}_{v \in V}$ and edge information $\{e_{uv}\}$ and computes output embeddings $\{\mathbf{h}^{(l)}_v\}_{v \in V}$ $\{\mathbf{x}^{(l)}_v\}_{v \in V}$ for nodes and coordinates respectively. One layer operation can be summarized as follows
\begin{equation}
\label{eq:egnn_update}
\begin{split}
    \mathbf{m}^{(l-1)}_{u \to v} &= \phi_e\left( \mathbf{h}_u^{(l-1)}, \mathbf{h}_v^{(l-1)}, \lVert \mathbf{x}^{(l-1)}_u - \mathbf{x}^{(l-1)}_v \rVert, e_{uv} \right) \\
    \mathbf{x}^{(l)}_v &= \mathbf{x}^{(l-1)}_v + C \sum_{u \neq v}\left( \mathbf{x}_v^{(l-1)} - \mathbf{x}_u^{(l-1)} \right) \phi_x(\mathbf{m}^{(l-1)}_{u \to v}) \\
    \mathbf{m}^{(l-1)}_v &= \sum_{u \neq v} \mathbf{m}^{(l-1)}_{u \to v} \\
    \mathbf{h}_v^{(l)} &= \phi_h(\mathbf{h}_v^{(l-1)}, \mathbf{m}^{(l-1)}_v),
\end{split}
\end{equation}
where $\phi_e, \phi_x$ and $\phi_h$ are learnable functions and $\{\mathbf{h}^{(0)}_v\}_{v \in V}$ are initialized to input node features, i.e. one-hot encodings of amino acid types and $\{\mathbf{x}^{(0)}_v\}_{v \in V}$ are initialized with 3D coordinates of residues' centers of mass.
We encode the protein as the average embedding of the nodes after $L$ layers:
\begin{equation}
    \mathbf{h}_P = \frac{1}{|V|}\sum_{v \in V}\mathbf{h}^{(L)}_v,
\end{equation}
where we use $L=3$ in our experiments. $\phi_e, \phi_x$ and $\phi_h$ are all implemented as concatenation of the input followed by a fully connected 2-layer MLP with hidden and output dimensions of $16$ and the SiLU activation function \citep{silu}.

We represent the ligand's unlabelled graph structure with the following four features: (i) number of nodes, (ii) number of rings, (iii) number of rotatable bonds and (iv) graph diameter.
All these features are standardized by subtracting the mean and dividing by the standard deviation, where these statistics were computed on the train set.
The standardized ligand features are subsequently concatenated with $\mathbf{h}_P$ and passed through 5 layers of a fully connected MLP with 50 hidden units, ReLU activation functions and 1 output unit. The model has a total of 15k trainable parameters.
\subsection{Training details}\label{app:scoring_model_training}
The dataset for training the scoring model consists of triples $(G^P, G^M, y)$ representing the protein, ligand and the binding affinity estimated with Vina respectively.
We randomly select 10\% of the protein pockets for validation and the remaining 90\% for training.
The model is trained to minimize the mean squared error with the Adam optimizer \citep{adam} with default parameters of its PyTorch implementation \citep{pytorch}.
The model was trained with batch size 128 for 10000 steps, after which we did not observe improvements in the validation loss.
Model training took approximately 45 minutes on an Apple M1 CPU.

\section{RDKit POSE GENERATION VARIABILITY}\label{app:pose_var}
As mentioned in Section \ref{seq:atom_sampler}, we use RDKit to generate 3D pose for a molecule given its molecular graph.
We now argue that there is substantial variation in the poses generated by RDKit for a fixed molecular graph.
Specifically, we performed the following experiment. We randomly sampled 1000 molecules from the ZINC250k dataset.
We subsequently generated 20 poses for each and computed their pairwise distances.
Where for two poses of the same molecule we define their distance as the lowest possible root-mean-square deviation (RMSD) (computed with \texttt{rdkit.AllChem.GetBestRMS(mol1, mol2)}).
We used the following pseudocode.
\begin{verbatim}
mols = random.sample(ZINC250k, k=1000)
variation = []
for mol in mols:
    poses = [generate_pose(mol) for i in range(20)]
    variation.append(average_pairwise_distance(poses))
mean(variation), std(variation)
\end{verbatim}
We found that the average pairwise distance between generated poses is 1.60 Å ± 0.46 Å.
This average value is higher than e.g. the lengths of C-C, C-H, or C-F bonds suggesting that there is considerable variation in the poses generated by RDKit and it is not merely performing rigid transformations of the same pose.
\section{LIGAND'S CENTER OF MASS PREDICTOR}\label{app:com_predictor}
As mentioned in Section \ref{seq:atom_sampler}, in order to accurately estimate the binding affinity for a protein-ligand pair with Vina, the ligand's 3D configuration must be in the vicinity of the protein 3D configuration.
The reason for that is that the scoring function defined in Eq. (\ref{eq:scoring_function}) vanishes for atom pairs that are far apart \citep{scoring_function}.
Consequently, gradient based optimization used to find the best 3D configuration (Eq. (\ref{eq:binding_affinity})) will return the original configuration, because all gradients are zero whenever the protein and ligand are too far apart.

To that end, we train a model, which predicts the center of mass of ligand's 3D configuration based on the protein pocket.
We subsequently use the model to place the ligands' 3D configurations computed with RDKit at the predicted center of mass.
As model parametrization we use the EGNN as described in Appendix \ref{app:vina_approx}.
The model output is defined as:
\begin{equation}
    \hat{\mathbf{x}} = \frac{1}{|V|}\sum_{v \in V}\mathbf{x}^{(L)}_v,
\end{equation}
where $\mathbf{x}^{(L)}_v$ is defined in Eq. (\ref{eq:egnn_update}).
As shown in \citep{egnn}, $\hat{\mathbf{x}}$ is then equivariant to translations and rotations as required for the center of mass predictor model.
In our experiments, we used $L=4$ layers and a hidden dimension of $16$. The model has 8k trainable parameters and \textit{both its number of parameters and runtime were included when reporting the results for SimpleSBDD models in Table \ref{tab:main_results}}.

To train the model, we used the training protein-ligand pairs of the CrossDocked2020 \citep{crossdocked} dataset described in Section \ref{sec:results}.
We further split them into train and validation sets to monitor overfitting using a 85-15 train-validation split.
Similarly to \citep{pocket2mol}, the split was performed to ensure that all validation protein pockets will have sequence similarity lower than 30\% to all training protein pockets.

The model was trained with Adam optimizer \citep{adam} with default parameters of its PyTorch implementation \citep{pytorch} to minimize the Euclidean distance between the predicted and ground truth centers of mass.
We used a batch size of $64$ and trained the model for 10000 steps until the validation loss converged to $\sim 1.15$Å.
Training took approximately 25 minutes on an Apple M1 CPU.

\section{ADDITIONAL MOLECULAR METRICS}\label{app:additional_metrics}
Here we provide additional metrics to complement the ones reported in Table \ref{tab:main_results}, which were omitted for presentation clarity.
Namely, we include the Lipinski's Rule of Five and LogP, which were defined in Section \ref{sec:prelims}.
For Lipinski's Rule of Five, $5$ is the highest possible value.
For LogP we do not indicate which method is the best as there is no straightforward criterion to compare LogP values.
In \citep{pocket2mol} it is noted that values in the $(-0.4, 5.6)$ interval are considered good drug candidates.
Note that for SimpleSBDD, the value of LogP is the closest to the midpoint of that interval $(2.6)$. None of the other baseline methods reported LogP or the Lipinski's Rule of Five.

\subsection{Novelty of generated molecules}
In Section \ref{sec:results} we report the novelty of all the methods as the average Tanimoto dissimilarity to the most similar molecule in the training set.
Since our method also uses ZINC250k dataset, we also measured the novelty when treating the ZINC250k dataset as the training set.
We differentiate between these two metrics as Novelty$_{\text{CD}}$ (treating CrossDocked2020 as the training set) and Novelty$_{\text{Z}}$ (treating ZINC250K as the training set).
We report the numbers in Table \ref{tab:additional_metrics}.

\begin{table}[h]
    \caption{Additional metrics. Error bars correspond to the standard deviation across test protein pockets.}
    \label{tab:additional_metrics}
    \centering
    \resizebox{0.7\linewidth}{!}{
    \begin{tabular}{L{3.1cm} c c c c}
    \toprule
         & Lipinski (↑) & LogP & Novelty$_{\text{CD}}$ (↑) & Novelty$_{\text{Z}}$ (↑) \\
    \midrule
         Test set & 4.34 \textcolor{gray}{± 1.14} &0.89 \textcolor{gray}{± 2.73} & - & -\\
         \cmidrule(lr){1-5}
        DiffSBDD (\citeyear{diffsbdd}) & 4.73 \textcolor{gray}{± 0.69}& 1.00 \textcolor{gray}{± 1.90}& \textbf{0.54}  \textcolor{gray}{± 0.14} & \textbf{0.52}  \textcolor{gray}{± 0.13}\\
        Pocket2Mol (\citeyear{pocket2mol}) & 4.90 \textcolor{gray}{± 0.35} & 1.71 \textcolor{gray}{± 1.98}&  0.45  \textcolor{gray}{± 0.16} & 0.46  \textcolor{gray}{± 0.15}\\
        FLAG (\citeyear{flag}) & 4.94 \textcolor{gray}{± 0.14} & 0.63 \textcolor{gray}{± 2.38}& 0.44 \textcolor{gray}{± 0.17} & 0.46 \textcolor{gray}{± 0.15}\\
        DrugGPS (\citeyear{drug-gps}) & 4.92 \textcolor{gray}{± 0.12} & 0.91 \textcolor{gray}{± 2.15}& 0.47 \textcolor{gray}{± 0.15} & 0.48 \textcolor{gray}{± 0.15}\\
        TargetDiff (\citeyear{target-diff}) & 4.59 \textcolor{gray}{± 0.83} & 1.37 \textcolor{gray}{± 2.37}& 0.47  \textcolor{gray}{± 0.14} & 0.46  \textcolor{gray}{± 0.13}\\
        DecompDiff (\citeyear{decompdiff}) & 4.64 \textcolor{gray}{± 0.90} & 1.57 \textcolor{gray}{± 2.12}& 0.52  \textcolor{gray}{± 0.13}& 0.48  \textcolor{gray}{± 0.12}\\
        D3FG (\citeyear{D3FG}) & 4.97 \textcolor{gray}{± NA} & 2.82 \textcolor{gray}{± NA}& - & -\\
        EQGAT-diff (\citeyear{EQGAT-diff}) & 4.66 \textcolor{gray}{± 0.72} & -& - & -\\
        \cmidrule(lr){1-5}
        SimpleSBDD  & 4.96 \textcolor{gray}{± 0.20} & 3.33 \textcolor{gray}{± 1.49}& 0.51  \textcolor{gray}{± 0.10} & 0.47  \textcolor{gray}{± 0.10}\\
        SimpleSBDD--$\mathcal{PO}$ & \textbf{5.00} \textcolor{gray}{± 0.03} & 3.24  \textcolor{gray}{± 0.79}& 0.50  \textcolor{gray}{± 0.09} & 0.44  \textcolor{gray}{± 0.09}\\
    \bottomrule
    \end{tabular}
    }
\end{table}
\section{BASELINE COMPARISON DETAILS}\label{app:baselines}
In this section we provide more details on the baseline comparison reported in Table \ref{tab:main_results}.

\paragraph{Pocket2Mol} We re-evaluated the method using samples generated with a checkpoint available in the official implementation and obtained results very close to the originally reported.

\paragraph{DiffSBDD} The authors did not report some of the metrics, so we generated the samples using checkpoints provided by the authors in the original implementation. However, when we tried to use the DiffSBDD-inpaint model variant (reportedly the best one) the sampling produces non-sensical molecules comprising of single atoms. We were able to generate proper samples using the DiffSBDD-cond model variant and this is what we report in all our comparisons.

\paragraph{FLAG \& DrugGPS} We used the numbers provided in the original publications. The only exception was Novelty (1 - \texttt{Sim.Train} in the original publications). We were able to compute Novelty with our code using the samples provided by the authors. The results differ significantly and we suspect that the similarity was computed differently, but we were unable to verify as the code repositories do not contain evaluation code.

\paragraph{TargetDiff \& DecompDiff} In the manuscripts, the authors use a different version of the Vina software, so we re-evaluate the methods ourselves. For TargetDiff, we use the samples provided by the authors and for DecompDiff, we generate them using the official implementation and provided checkpoints.

\paragraph{EQGAT-diff} We copy the numbers from the original publication. We were unable to verify or reproduce the results, because evaluation code is not provided. We were also unable to run the method ourselves, because as the authors say in the official code repository: \textit{Note that currently we do not provide the datasets, hence the code is currently just for reviewing and how the model and training runs are implemented.} Also, the trained model checkpoints are provided on request, but only for unconditional generation models, not condintioned on the protein pocket.

\paragraph{D3FG} We copy the results from the original publication, but were unable to verify if the metrics were computed in the same way as there is no implementation available.

\section{SIMPLESBDD WITH PROPERTY OPTIMIZATION}\label{app:sbdd-po}
In this section we provide additional details on the property optimization capabilities of SimpleSBDD described in Section \ref{sec:add_eval}. Specifically, we show in Table \ref{tab:po_results} that by the introduction of the additional molecular property procedure, we improve all of the metrics without sacrificing the diversity of the generated molecules. This comes at the additional computational cost, which is still orders of magnitude lower than baselines.

\begin{table}[]
    \caption{\textbf{Our method allows for simultaneous optimization of binding and other properties}. Comparison of our method w/ and w/o property optimization.
    With explicit property optimization, we can significantly improve all metrics without sacrificing binding affinity. The increased sampling time is still an order of magnitude lower than baseline methods even though running exclusively on CPU. Error bars correspond to the standard deviation across test protein pockets.}
    \label{tab:po_results}
    \centering
    \resizebox{0.75\linewidth}{!}{
    \begin{tabular}{L{3cm} C{2cm} C{1.8cm} c c c C{1.7cm} c}
    \toprule
          & Vina Score (kcal/mol, ↓) & High Affinity (↑) &  QED (↑) & SA (↑) & Diversity (↑) & \#Params (↓) & Time (s, ↓)  \\
    \midrule
         Test set & -6.99 \textcolor{gray}{± 2.16} &  - & 0.48 \textcolor{gray}{± 0.21} & 0.73 \textcolor{gray}{± 0.14} & - &- & - \\
         \cmidrule(lr){1-8}
        SimpleSBDD& -7.78 \textcolor{gray}{± 1.47} & 0.71 \textcolor{gray}{± 0.34} & 0.61 \textcolor{gray}{± 0.18} & 0.69 \textcolor{gray}{± 0.09}& 0.68 \textcolor{gray}{± 0.06}	 &23K & \textbf{3.9}$^\dagger$ \textcolor{gray}{± 0.9} \\
        SimpleSBDD--$\mathcal{PO}$& \textbf{-7.98} \textcolor{gray}{± 1.46} & \textbf{0.75} \textcolor{gray}{± 0.35} & \textbf{0.80} \textcolor{gray}{± 0.10} & \textbf{0.73} \textcolor{gray}{± 0.08}& 0.67 \textcolor{gray}{± 0.06}	 &23K & 115$^\dagger$ \textcolor{gray}{± 11} \\
    \bottomrule
    \end{tabular}
    }
\end{table}

\subsection{Ablation study}
We analyzed how the number of proposal molecules $K$ impacts the performance.
In Section \ref{sec:add_eval} we chose $K=50$ and now investigate what happens for $K=10, 20, 50$.
For this experiment, due to computational constraints, we considered a random subset of 20\% of the test protein targets and generated 50 molecules per each (as opposed to the standard of 100).
Note that the runtime of property optimization scales linearly with $K$.

We found that increasing $K$ improves both QED and SA as expected, as the model optimizes for a mixture of these two properties.
Furthermore, we found that the choice of $K$ does not seem to affect the predicted binding affinity, which strengthens our point that the optimization of atom types for various properties for a fixed unlabeled graph structure does not degrade the predicted binding affinity.
Please see Table \ref{tab:po-ablation} for more details.

\begin{table}[]
    \caption{Increasing the number of proposals $K$ improves molecular metrics without degrading the predicted binding affinity.}
    \label{tab:po-ablation}
    \centering
    \resizebox{0.95\linewidth}{!}{
    \begin{tabular}{l c c c c c}
    \toprule
          & Vina Score (↓) & High Affinity (↑) &  QED (↑) & SA (↑) & Diversity (↑) \\
    \midrule
         Test set & -6.99 \textcolor{gray}{± 2.16} &  - & 0.48 \textcolor{gray}{± 0.21} & 0.73 \textcolor{gray}{± 0.14} & -\\
         \cmidrule(lr){1-6}
         SimpleSBDD--$\mathcal{PO} \ (K=10)$ & -7.935 \textcolor{gray}{± 1.469} & 0.705 \textcolor{gray}{± 0.356} & 0.760 \textcolor{gray}{± 0.107} & 0.715 \textcolor{gray}{± 0.078} & 0.651 \textcolor{gray}{± 0.054} \\ 
         SimpleSBDD--$\mathcal{PO} \ (K=20)$ & -7.934 \textcolor{gray}{± 1.465} & 0.724 \textcolor{gray}{± 0.352} & 0.781 \textcolor{gray}{± 0.100} & 0.724 \textcolor{gray}{± 0.079} & 0.651 \textcolor{gray}{± 0.055} \\ 
         SimpleSBDD--$\mathcal{PO} \ (K=50)$ & -7.939 \textcolor{gray}{± 1.462} & 0.715 \textcolor{gray}{± 0.375} & 0.796 \textcolor{gray}{± 0.094} & 0.734 \textcolor{gray}{± 0.077 } & 0.652 \textcolor{gray}{± 0.055} \\ 
    \bottomrule
    \end{tabular}
    }
\end{table}

\section{DRUG REPURPOSING}\label{app:drug_repurposing}
As mentioned in Section \ref{sec:add_eval}, our scoring model can be used for drug repurposing.
We found that for each protein pocket it takes around 2 seconds to find 100 diverse molecules with much higher predicted binding affinity and other improved molecular metrics.
Furthermore, we repeated the experiment with ChEMBL dataset \citep{chembl1, chembl2} and found that $g_\theta$ generalizes well across molecular databases.
See Table \ref{tab:repurposing} for more details

\begin{table}
    \caption{\textbf{Scoring model generalizes across datasets and quickly finds diverse high quality molecules.} The scoring model $g_\theta$ used for drug repurposing. We score the molecular databases: ZINC250k and ChEMBL respectively with $g_\theta$. The best scoring ones are compared with the reference molecule. Error bars correspond to the standard deviation across test protein pockets.}
    \label{tab:repurposing}
    \centering
    \resizebox{0.75\linewidth}{!}{
    \begin{tabular}{L{2.5cm} C{2.2cm} c C{1.8cm} c c}
    \toprule
         & Vina Score (↓) & QED (↑) & SA (↑)& Diversity (↑) & Time (s, ↓)\\
    \midrule
         Test set & -6.99 \textcolor{gray}{± 2.16} & 0.48 \textcolor{gray}{± 0.21} & 0.73 \textcolor{gray}{± 0.14} & - & - \\
        \cmidrule(lr){1-6}
        $g_\theta$ (ZINC) & -7.82 \textcolor{gray}{± 1.47} & \textbf{0.66} \textcolor{gray}{± 0.15} & \textbf{0.78} \textcolor{gray}{± 0.08} & 0.68 \textcolor{gray}{± 0.05} & \textbf{1.8} \textcolor{gray}{± 0.4}\\
        $g_\theta$ (ChEMBL) & \textbf{-8.03} \textcolor{gray}{± 1.56} & 0.56 \textcolor{gray}{± 0.15} & \textbf{0.79} \textcolor{gray}{± 0.08} & 0.67 \textcolor{gray}{± 0.07} & 2.6 \textcolor{gray}{± 1.0}\\
    \bottomrule
    \end{tabular}
    }
\end{table}

\section{REPRESENTATION LIMITS OF GNNS FOR SBDD}\label{app:repr_limits}
In this section we prove claims made in Section \ref{sec:th_repr}. We begin with the notion of indistinguishability for LU(3D)-GNNs. Recall that a layer of a LU-GNN network can be summarized as follows
\begin{equation}
\begin{split}
    m_{u \to v}^{(l-1)} &= \phi(h_u^{(l-1)}, e_{uv}) \\
    \Tilde{h}_v^{(l-1)} &= \mbox{AGG} \lbrace m_{u \to v}^{(l-1)} \: | \: u \in N(v) \rbrace \\
    h_v^{(l)} &= \mbox{COMBINE} \lbrace h_v^{(l-1)}, \Tilde{h}_v^{(l-1)} \rbrace,
\end{split}
\end{equation}
where $h_v^{(l)}$ is the feature vector of node $v$ after applying $l$ layers and $h_v^{(0)}=h_v$ are $v$'s input features. A layer of a LU3D-GNN is defined analogously with messages depending on distances to neighbouring nodes:
\begin{equation}
    m_{u \to v}^{(l-1)} = \phi(h_u^{(l-1)}, e_{uv}, \lVert x_u - x_v \rVert),
\end{equation}
where $x_v$ are 3D coordinates of $v$.
Before we discuss indistinguishability, we formally define it.
\begin{definition}[Indistinguishability]
    We say that graphs $G_1=(V, E_1)$, $G_2=(W, E_2)$ are indistinguishable by LU(3D)-GNNs if there exist node orderings $(v_1, v_2, \dots, v_N)$ and $(w_1, w_2, \dots, w_N)$ of $V$ and $W$ respectively s.t. for any choice of $\phi, \mbox{AGG}, \mbox{COMBINE}$ and network depth $L$: $$h_{v_i}^{(l)}=h_{w_i}^{(l)} \quad \forall i=1, \dots, N, l=1, \dots, L$$
    We denote indistinguishability by $G_1 \equiv G_2$ for LU-GNNs and by $G_1 \equiv_{3D} G_2$ for LU3D-GNNs.
\end{definition}
In other words, $G_1, G_2$ are not distinguishable by LU(3D)-GNNs, if we can pair up their nodes in such a way that corresponding embeddings remain identical in both graphs after any number of layers of any LU(3D)-GNN.
We first note that, since LU3D-GNNs are strictly more expressive than LU-GNNs, it holds that:
\begin{equation}
    G_1 \equiv_{3D} G_2 \implies G_1 \equiv G_2.
\end{equation}
A useful construct for determining indistinguishability is that of computation trees \citep{comp_trees}:
\begin{definition}[Computation trees]
    For a graph $G = (V, E)$, the computation tree of node $v$ is defined recursively as follows:
    \begin{enumerate}
        \item $T_G^{(1)}(v)=h_v$, where $h_v$ is $v$'s feature vector,
        \item for $l \ge 2$, the root of $T_G^{(l)}(v)$ is $h_v$, its children are $\{T_G^{(l-1)}(u) \: | \: u \in N(v) \}$ and the edge connecting the root to $T_G^{(l-1)}(u)$ inherits features $e_{vu}$ from $G$.
    \end{enumerate}
We denote the multiset of depth-$L$ computation trees of $G$ as $\{T_G^{(L)}(v) \}_{v \in V}$.
\end{definition}
Since the multiset of depth-$L$ computation trees represents all information LU-GNNs can extract \citep{comp_trees}, we see that
\begin{equation}
    G_1 \equiv G_2 \Leftrightarrow \: \forall_L \: \: \{T_{G_1}^{(L)}(v) \}_{v \in V_1} = \{T_{G_2}^{(L)}(v) \}_{v \in V_2}
\end{equation}
and analogously for $\equiv_{3D}$ whenever the edges of the computation trees carry the information about the distances between nodes. Using the computation tree terminology, we can now prove Lemma \ref{lemma:3d_single_body}.
\begin{figure}[t]
    \centering
    \includegraphics[width=0.95\textwidth]{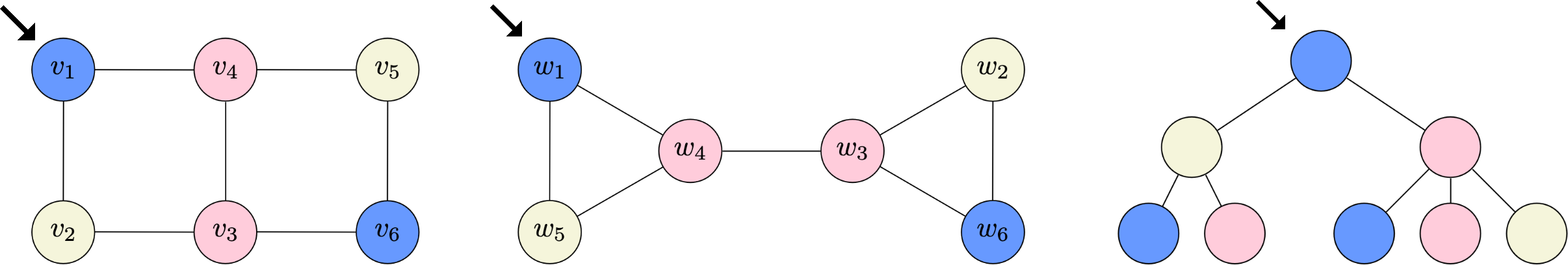}
    \caption{\textbf{Construction for Lemma \ref{lemma:3d_single_body}.} Left and Middle: A pair of non-isomorphic connected graphs differing in all graph properties listed in Lemma \ref{lemma:3d_single_body}. Right: An example computation tree of depth three, which is identical for nodes $v_1$ and $w_1$.}
    \label{fig:3d_single_body}
\end{figure}
We recall it for completeness:
\singlebody*
\begin{proof}
In Figure \ref{fig:3d_single_body} (left and middle), we provide an example of a pair of non-isomorphic connected graphs, which differ in all mentioned graph properties. Colors indicate feature vectors: nodes with the same color have identical features. Edge features can be defined as any function of their endpoints, i.e. $e_{uv}=f(h_u, h_v)$ for any $f$. All edges are of equal length. For clarity of presentation, the graphs are presented in 2D, but any 3D configuration preserving equal edge lengths could be used instead.

In the figure, we show node orderings such that $T_{G_1}^{(l)}(v_i)=T_{G_2}^{(l)}(w_i)$ for all $i=1, \dots, 6$ and $l \ge 1$. As an example, we show $T_{G_1}^{(3)}(v_1)$ in Figure \ref{fig:3d_single_body} (Right), which is identical to $T_{G_2}^{(3)}(w_1)$.
\end{proof}

We now move on to proving Proposition \ref{prop:multi_body}. We recall it for completeness:
\multibody*
\begin{proof}
    We begin with proving $\ref{prop:1}$. We will abuse notation slightly and for a graph $G=(V, E)$ write $v \in V$ and $v \in G$ interchangeably. We will also write $N_G(v)$ to denote the neighbourhood of $v$ in $G$ whenever we want to emphasize which graph we are considering. Let $C_i=\mathcal{C}(P, G_i)$ for notation brevity and let $(v_1, \dots, v_N)$ and $(w_1, \dots, w_N)$ be node orderings of $G_1$ and $G_2$ respectively such that $T^{(l)}_{G_1}(v_i)=T^{(l)}_{G_2}(w_i)$, for all $i=1, \dots, N$ and $l \ge 1$. We can find such an ordering, because $G_1 \equiv G_2$. Let now $(p_1, \dots, p_K)$ be any ordering of nodes of $P$. We will show that for $(u_1, \dots, u_{N + K})=(p_1, \dots, p_K, v_1, \dots, v_N)$ and $(u'_1, \dots, u'_{N + K})=(p_1, \dots, p_K, w_1, \dots, w_N)$ the following holds:
    \begin{equation}
    \label{eq:main_claim}
        T^{(l)}_{C_1}(u_i) = T^{(l)}_{C_2}(u'_i) \quad \forall l \ge 1.
    \end{equation}
    We start with an auxiliary fact that for all $i, j = 1, \dots, N$ and any $l \ge 1$ the following holds:
    \begin{equation}
    \label{eq:aux_ind}
        \left( T^{(l)}_{G_1}(v_i) = T^{(l)}_{G_1}(v_j) \right) \Rightarrow T^{(l)}_{C_1}(v_i) = T^{(l)}_{C_1}(v_j).
    \end{equation}
    We show (\ref{eq:aux_ind}) by induction. The base case of $l=1$ is trivial since $T_G^{(1)}(v)=h_v$ for any $v$ and $G$. Assume now that (\ref{eq:aux_ind}) holds for $k=1, \dots, l$ and $T^{(l+1)}_{G_1}(v_i) = T^{(l+1)}_{G_1}(v_j)$. We will show that $T^{(l+1)}_{C_1}(v_i) = T^{(l+1)}_{C_1}(v_j)$. We first note that the roots coincide:
    \begin{equation}
    \label{eq:root_aux}
    \mbox{ROOT}(T^{(l+1)}_{C_1}(v_i))=h_{v_i}=h_{v_j}=\mbox{ROOT}(T^{(l+1)}_{C_1}(v_j))
    \end{equation}
    from the base case. Furthermore
    \begin{equation}
    \label{eq:childred_aux}
    \begin{split}
    \mbox{CHILDREN}\left(T^{(l+1)}_{C_1}(v_i)\right) &= \{ T^{(l)}_{C_1}(v) \: | \: v \in N_{C_1}(v_i) \} \\
    & = \{ T^{(l)}_{C_1}(v) \: | \: v \in N_{G_1}(v_i) \} \cup \{ T^{(l)}_{C_1}(p) \: | \: p \in P) \}
    \end{split}
    \end{equation}
    Now since $T^{(l+1)}_{G_1}(v_i) = T^{(l+1)}_{G_1}(v_j)$, it holds that:
    \[
    \{ T^{(l)}_{G_1}(v) \: | \: v \in N_{G_1}(v_i)\} =  \{ T^{(l)}_{G_1}(v) \: | \: v \in N_{G_1}(v_j) \}
    \]
    and therefore by our induction assumption:
    \[
    \{ T^{(l)}_{C_1}(v) \: | \: v \in N_{G_1}(v_i)\} =  \{ T^{(l)}_{C_1}(v) \: | \: v \in N_{G_1}(v_j) \}
    \]
    and subsequently, continuing (\ref{eq:childred_aux}):
    \begin{equation}
    \label{eq:children_aux_ctd}
            \begin{split}
    \mbox{CHILDREN}\left(T^{(l+1)}_{C_1}(v_i)\right) &= \{ T^{(l)}_{C_1}(v) \: | \: v \in N_{C_1}(v_i) \} \\
    & = \{ T^{(l)}_{C_1}(v) \: | \: v \in N_{G_1}(v_i) \} \cup \{ T^{(l)}_{C_1}(p) \: | \: p \in P) \} \\
    &= \{ T^{(l)}_{C_1}(v) \: | \: v \in N_{G_1}(v_j) \} \cup \{ T^{(l)}_{C_1}(p) \: | \: p \in P) \} \\
    &=  \{ T^{(l)}_{C_1}(v) \: | \: v \in N_{C_1}(v_j) \} \\
    & =\mbox{CHILDREN}\left(T^{(l+1)}_{C_1}(v_j)\right).
    \end{split}
    \end{equation}
    Combining (\ref{eq:root_aux}) and (\ref{eq:children_aux_ctd}) we have shown that $T^{(l+1)}_{C_1}(v_i) = T^{(l+1)}_{C_1}(v_j)$ and therefore proven (\ref{eq:aux_ind}). We now proceed to show (\ref{eq:main_claim}) by proving a stronger statement, i.e. that for all $l \ge 1$ and all $i,j=1,\dots, N, k=1,\dots, K$ the following hold:
    \begin{align}
        (1) \: \: & \left(T^{(l)}_{G_1}(v_i) = T^{(l)}_{G_2}(w_j)\right) \Rightarrow T^{(l)}_{C_1}(v_i) = T^{(l)}_{C_2}(w_j) \label{eq:main_ind_1}\\
        (2) \: \: &T^{(l)}_{C_1}(p_k) = T^{(l)}_{C_2}(p_k) \label{eq:main_ind_2}
    \end{align}
    Again, we proceed with a proof by induction and note that the base case of $l=1$ is trivial, because all computation trees are just features of roots. Assume now that (\ref{eq:main_ind_1}) and (\ref{eq:main_ind_2}) hold for $k=1,\dots, l$. We will show they also hold for $k=l+1$. Let us first consider any $k=1,\dots, K$ and note that
    \begin{equation}
    \label{eq:main_ind_p_root}
        \mbox{ROOT}\left(T^{(l+1)}_{C_1}(p_k)\right) = h_{p_k} =\mbox{ROOT}\left(T^{(l+1)}_{C_2}(p_k)\right).
    \end{equation}
    And for the children:
    \begin{equation}
    \begin{split}
    \label{eq:main_ind_p_children}
    \mbox{CHILDREN}\left(T^{(l+1)}_{C_1}(p_k)\right) & = \{ T^{(l)}_{C_1}(v) \: | \: v \in N_{C_1}(p_k) \} \\
    &= \{ T^{(l)}_{C_1}(p) \: | \: p \in N_{P}(p_k) \} \cup \{ T^{(l)}_{C_1}(v) \: | \: v \in G_1 \}.
    \end{split}
    \end{equation}
    Since $G_1 \equiv G_2$, it holds that $\{ T_{G_1}^{(l)}(v)\}_{v \in G_1}=\{ T_{G_2}^{(l)}(w)\}_{w \in G_2}$ and therefore, by the induction assumption (\ref{eq:main_ind_1}):
    \begin{equation}
    \label{eq:main_ind_2_step_1}
        \{ T^{(l)}_{C_1}(v) \: | \: v \in G_1\} = \{ T^{(l)}_{C_2}(w) \: | \: w \in G_2 \}.
    \end{equation}
    Furthermore, by the induction assumption (\ref{eq:main_ind_2}), it holds that
    \begin{equation}
    \label{eq:main_ind_2_step_2}
        \{ T^{(l)}_{C_1}(p) \: | \: p \in N_{P}(p_k) \} = \{ T^{(l)}_{C_2}(p) \: | \: p \in N_{P}(p_k)\}.
    \end{equation}
    Combining (\ref{eq:main_ind_2_step_1}) and (\ref{eq:main_ind_2_step_2}) yields:
    \begin{equation}
    \begin{split}
    \label{eq:main_ind_p_children_ctd}
    \mbox{CHILDREN}\left(T^{(l+1)}_{C_1}(p_k)\right) & = \{ T^{(l)}_{C_1}(v) \: | \: v \in N_{C_1}(p_k) \} \\
    &= \{ T^{(l)}_{C_1}(p) \: | \: p \in N_{P}(p_k) \} \cup \{ T^{(l)}_{C_1}(v) \: | \: v \in G_1\} \\
    &= \{ T^{(l)}_{C_2}(p) \: | \: p \in N_{P}(p_k) \: \} \cup \{ T^{(l)}_{C_2}(w) \: | \: w \in G_2\} \\
    &= \{ T^{(l)}_{C_2}(w) \: | \: w \in N_{C_2}(p_k) \: \} \\
    &= \mbox{CHILDREN}\left(T^{(l+1)}_{C_2}(p_k)\right)
    \end{split}
    \end{equation}
    and thus (\ref{eq:main_ind_p_root}) and (\ref{eq:main_ind_p_children_ctd}) imply 
    \begin{equation}
    \label{eq:ind_main_2_qed}
        T^{(l+1)}_{C_1}(p_k) = T^{(l+1)}_{C_2}(p_k).
    \end{equation}
    We will now show that $T^{(l+1)}_{C_1}(v_i)=T^{(l+1)}_{C_2}(w_i)$ for all $i$. Consider now any $i=1, \dots, N$ and note that:
    \begin{equation}
    \label{eq:main_ind_g_root}
        \mbox{ROOT}\left(T^{(l+1)}_{C_1}(v_i)\right) = h_{v_i} = h_{w_i} =\mbox{ROOT}\left(T^{(l+1)}_{C_2}(w_i)\right).
    \end{equation}
    Moreover:
    \begin{equation}
    \label{eq:main_ind_g_children}
        \begin{split}
            \mbox{CHILDREN}\left(T^{(l+1)}_{C_1}(v_i)\right) &= \{ T^{(l)}_{C_1}(v) \: | \: v \in N_{C_1}(v_i) \} \\
            &= \{ T^{(l)}_{C_1}(v) \: | \: v \in N_{G_1}(v_i) \} \cup \{ T^{(l)}_{C_1}(p) \: | \: p \in P \}.
        \end{split}
    \end{equation}
    Since $G_1 \equiv G_2$, it holds that $T_{G_1}^{(l+1)}(v_i) = T_{G_2}^{(l+1)}(w_i)$ and therefore \newline $\{ T_{G_1}^{(l)}(v) \: |  \: v \in N_{G_1}(v_i)\} = \{ T_{G_2}^{(l)}(w) \: | \: w \in N_{G_2}(w_i)\}$, so using the induction assumption (\ref{eq:main_ind_1}):
    \begin{equation}
    \label{eq:main_ind_1_step_1}
        \{ T^{(l)}_{C_1}(v) \: | \: v \in N_{G_1}(v_i) \} = \{ T^{(l)}_{C_2}(w) \: | \: w \in N_{G_2}(w_i) \}.
    \end{equation}
    On the other hand, from the induction assumption (\ref{eq:main_ind_2}):
    \begin{equation}
    \label{eq:main_ind_1_step_2}
    \{ T^{(l)}_{C_1}(p) \: | \: p \in P \} = \{ T^{(l)}_{C_2}(p) \: | \: p \in P \},
    \end{equation}
    so from (\ref{eq:main_ind_1_step_1}) and (\ref{eq:main_ind_1_step_2}), we have:
        \begin{equation}
    \label{eq:main_ind_g_children_ctd}
        \begin{split}
            \mbox{CHILDREN}\left(T^{(l+1)}_{C_1}(v_i)\right) &= \{ T^{(l)}_{C_1}(v) \: | \: v \in N_{C_1}(v_i) \} \\
            &= \{ T^{(l)}_{C_1}(v) \: | \: v \in N_{G_1}(v_i) \} \cup \{ T^{(l)}_{C_1}(p) \: | \: p \in P \} \\
            &= \{ T^{(l)}_{C_2}(w) \: | \: w \in N_{G_2}(w_i) \} \cup \{ T^{(l)}_{C_2}(p) \: | \: p \in P \} \\
            &= \{ T^{(l)}_{C_2}(w) \: | \: v \in N_{C_2}(w_i) \} \\
            &= \mbox{CHILDREN}\left(T^{(l+1)}_{C_2}(w_i)\right).
        \end{split}
    \end{equation}
    Using (\ref{eq:main_ind_g_root}) and (\ref{eq:main_ind_g_children_ctd}), we have shown that for all $i=1\dots, N$
    \begin{equation}
    \label{eq:ind_main_2_qed_almost}
    T^{(l+1)}_{C_1}(v_i) = T^{(l+1)}_{C_2}(w_i).
    \end{equation}
    Now take any $i, j=1, \dots, N$, such that $T^{(l+1)}_{G_1}(v_i) = T^{(l+1)}_{G_2}(w_j)$. Since $G_1 \equiv G_2$, it holds that $T^{(l+1)}_{G_1}(v_j) = T^{(l+1)}_{G_2}(w_j)=T^{(l+1)}_{G_1}(v_i)$ and thus:
    \begin{equation}
    \label{eq:ind_main_1_qed}
        T_{C_1}^{(l+1)}(v_i) \stackrel{(\ref{eq:aux_ind})}{=} T_{C_1}^{(l+1)}(v_j) \stackrel{(\ref{eq:ind_main_2_qed_almost})}{=} T_{C_2}^{(l+1)}(w_j).
    \end{equation}
    Putting together (\ref{eq:ind_main_1_qed}) and (\ref{eq:ind_main_2_qed}) concludes the induction step and therefore the proof of (\ref{eq:main_ind_1}) and (\ref{eq:main_ind_2}). An immediate consequence of (\ref{eq:main_ind_1}) is that for all $i=1,\dots, N$ and $l \ge 1$:
    \begin{equation}
        T^{(l)}_{C_1}(v_i) = T^{(l)}_{C_2}(w_i),
    \end{equation}
    which together with (\ref{eq:main_ind_2}) show that for all $i=1,\dots, K+N$ and $l \ge 1$:
    \begin{equation}
        \label{eq:final_qed}
        T_{C_1}^{(l)}(u_i) = T_{C_2}^{(l)}(u'_i)
    \end{equation}
    and thus $C_1 \equiv C_2$, concluding the proof of $\ref{prop:1}$.

    To show $\ref{prop:2}$, we take $G_1$ and $G_2$ to be molecular graphs corresponding to SMILES representations: C1CCC2CCCCC2C1 and C1CCC(C1)C2CCCC2 respectively (see Figure \ref{fig:identical_mols}).
    All heavy atoms are carbon and all bonds are single, so all node and edge features are identical.
    All bonds are also of equal length.
    We also take an arbitrary protein graph $P$ to define complexes $C_1=\mathcal{C}(P, G_1)$, $C_2=\mathcal{C}(P, G_2)$.
    \begin{figure}
        \centering
        \includegraphics[width=0.5\textwidth]{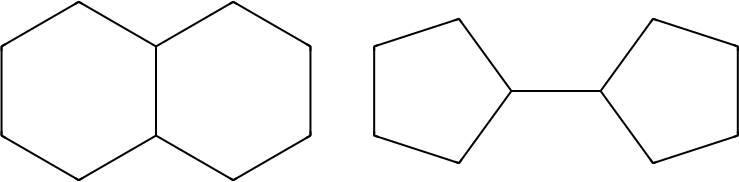}
        \caption{Molecular graphs indistinguishable by LU-GNNs.}
        \label{fig:identical_mols}
    \end{figure}

    The intra-ligand LU3D-GNN produces identical sets of embeddings for $G_1$ and $G_2$, which can be shown in the same way as in Lemma \ref{lemma:3d_single_body}. Similarly for intra-ligand LU-GNNs since $G_1 \equiv_{3D} G_2 \Rightarrow G_1 \equiv G_2$.

    The induced subgraphs of $C_1$ and $C_2$ corresponding to nodes of $P$ are identical and therefore any intra-protein GNN will produce identical sets of embeddings for both.

    Now since $G_1 \equiv G_2$, it follows from $\ref{prop:1}$ that $C_1 \equiv C_2$. Therefore inter protein-ligand LU-GNN will produce identical embeddings for $C_1$ and $C_2$.
\end{proof}

\subsection{Persistent Homology}
Our claims in Propostion \ref{prop:multi_body} also hold when the LUGNN model is additionally equipped with persistent homology (PH) information, which makes it strictly more expressive \citep{immonen2023going}.
This is a strict generalization, because PH computes global topological information such as the number of connected components or the number of independent cycles that LUGNNs cannot compute.

Specifically, we can show that if the two ligand graphs are indistinguishable by LUGNNs+PH then the complex graphs are also indistinguishable with LUGNNs+PH.

\paragraph{Proof}

Graphs $G_1$ and $G_2$ are indistinguishable by LUGNN+PH if and only if: 1) $G_1$ and $G_2$ are  indistinguishable by LUGNN, and 2) $G_1$ and $G_2$ have the same Betti numbers $\beta_0$, $\beta_1$, where
\begin{itemize}
    \item $\beta_0(G)$ is the number connected components of $G$, and
    \item $\beta_1(G)=\beta_0(G) + |E| - |V|$, where $G=(V, E)$, i.e. $V$ is the set of vertices and $E$ is the set of edges.
\end{itemize}

We now assume that $G_1$ and $G_2$ are indistinguishable by LUGNNs and $\beta_0(G_1)=\beta_0(G_2)$ and $\beta_1(G_1)=\beta_1(G_2)$.

In proposition \ref{prop:multi_body} we already showed that the complex graphs $C_1$ and $C_2$ are indistinguishable by LUGNNs, so it suffices to show that $C_1$ and $C_2$ have the same Betti numbers.

From our definition of the complex graph we see that it is always connected, i.e. has a single connected component, and thus $\beta_0(C_1)=\beta_0(C_2)=1$.

To evaluate the 1-st order Betti numbers we denote $G_1=(V_1, E_1)$, $G_2=(V_2, E_2)$, $C_i=\mathcal{C}(P, G_i)$ for some arbitrary protein graph $P=(V_P, E_P)$ using the notation from the paper.

Since $G_1$ and $G_2$ are indistinguishable by LUGNNs+PH, we must have $|V_1|=|V_2|$ and $|E_1|=|E_2|$. Therefore

\begin{align*}
\beta_1(C_1) &= \beta_0(C_1) + |E_P| + |E_1| + |E_P|\cdot|E_1| - |V_P|-|V_1| \\
& =1 + |E_P| + |E_2| + |E_P|\cdot|E_2| - |V_P|-|V_2| \\
& =\beta_1(C_2)    
\end{align*}

and thus $C_1$ and $C_2$ are indistinguishable by LUGNNs+PH.

\section{EMPIRICAL DEMONSTRATION OF THE THEORETICAL ANALYSIS}\label{app:real-world-demo}
In Section \ref{sec:th_repr} we provide a theoretical analysis of the expressivity of Graph Neural Networks (GNNs) and we show that there are inherent limitations in what a GNN can represent for 2-body systems like protein-ligand complexes.
As part of our contribution, we defined a scoring model, which estimates the docking scores based on the unlabelled graph structure and the target protein (Section \ref{sec:fastsbdd}).
Here, we provide an example of a protein pocket and two ligands such that:
\begin{itemize}
    \item They have different unlabelled graph structures
    \item They differ in the predicted docking scores to the given protein pocket
    \item They are identical from the perspective of LUGNNs
\end{itemize}

The protein has an identifier
\newline \texttt{BAZ2A\_HUMAN\_1795\_1898\_0/5mgl\_A\_rec\_5mgl\_7mu\_lig\_tt\_min\_0\_pocket10.pdb} and the two ligands are
\begin{itemize}
    \item SMILES: C1CCC2CCCCC2C1 (two 6-rings, zero 5-rings); Vina score -5.4
    \item SMILES: C1CCC(C1)C2CCCC2 (zero 6-rings, two 5-rings); Vina score -5.7
\end{itemize}

Even though these two ligands have different binding scores to the given protein, our analysis shows that if we defined the scoring model as a GNN, it would always predict the same value for both of them.
See Figure \ref{fig:wl_example}

\begin{figure}[h]
     \centering
     \begin{subfigure}[t]{0.48\textwidth}
         \centering
         \includegraphics[width=\textwidth]{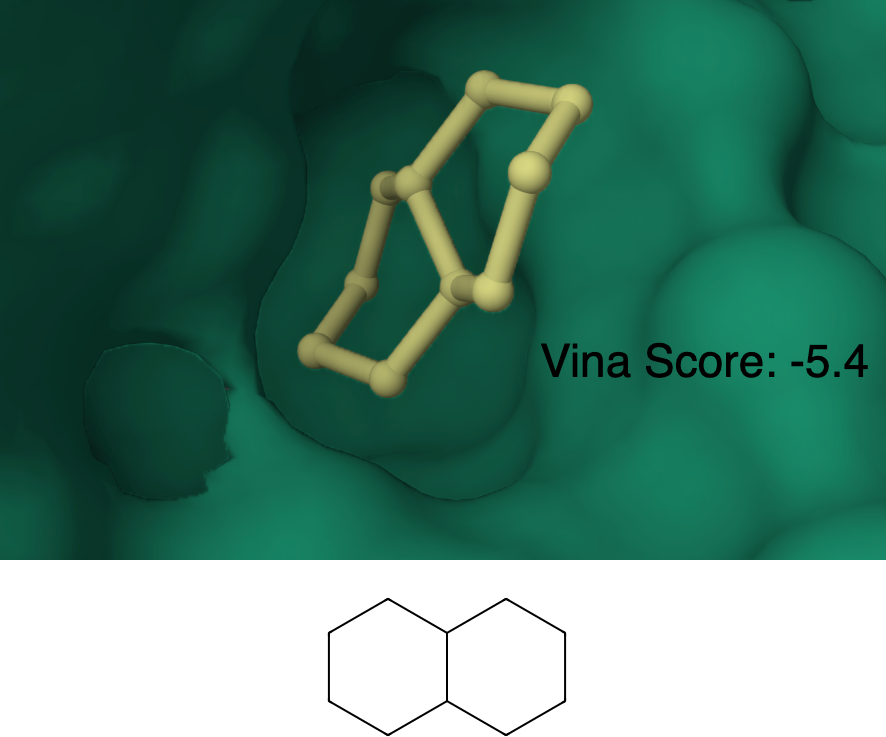}
         \caption{SMILES: C1CCC2CCCCC2C1}
         \label{fig:y equals x}
     \end{subfigure}
     \hfill
     \begin{subfigure}[t]{0.48\textwidth}
         \centering
         \includegraphics[width=\textwidth, height=159pt]{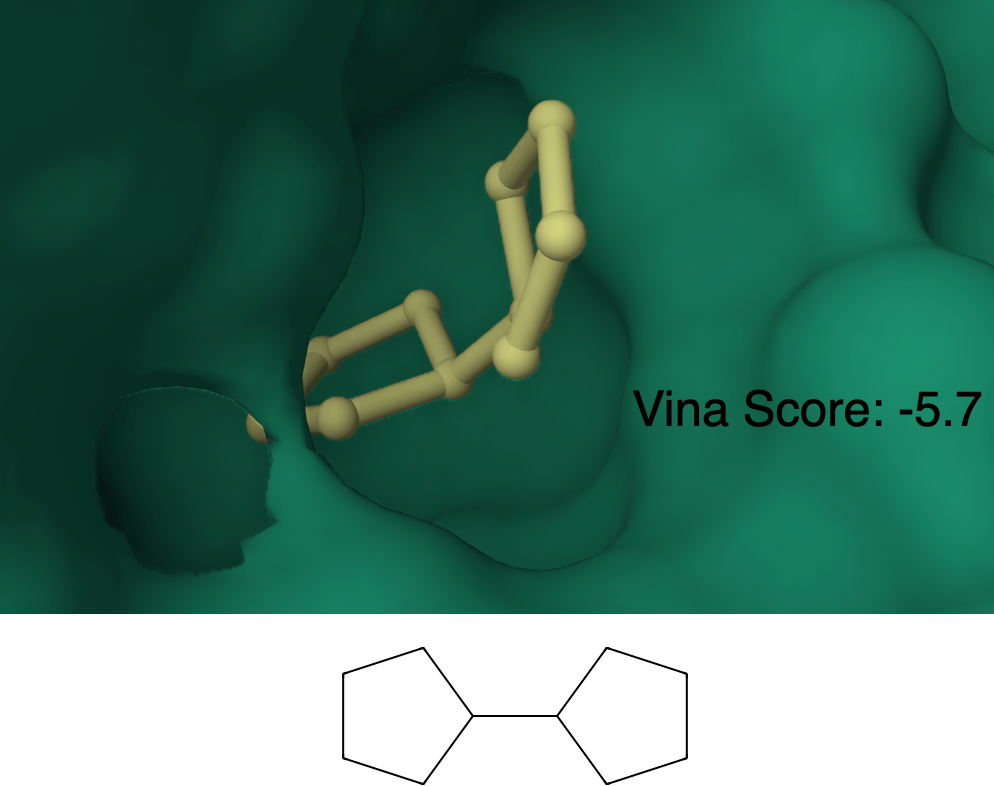}
         \caption{SMILES: C1CCC(C1)C2CCCC2}
         \label{fig:three sin x}
     \end{subfigure}
        \caption{An example of two different molecules with different binding affinities, which are identical from LU-GNN perspective.}
        \label{fig:wl_example}
\end{figure}
\section{LICENSES}\label{app:licences}
\paragraph{Datasets}
\begin{enumerate}
    \item CrossDocked2020 \citep{crossdocked}: GPL-2.0 license
    \item ZINC250k \citep{zinc}: GPL-3.0+ license
    \item ChEMBL \citep{chembl1, chembl2}: CC BY-SA 3.0
\end{enumerate}
\paragraph{Pre-trained models}
\begin{enumerate}
    \item MoFlow \citep{moflow}: MIT License
    \item DiffSBDD \citep{diffsbdd}: MIT License
    \item Pocket2Mol \citep{pocket2mol}: MIT License
    \item TargetDiff \citep{target-diff}: MIT License
    \item DecompDiff \citep{decompdiff}: CC BY-NC 4.0
    \item RGA \citep{rga}: License status unclear
    \item 3D-MCTS \citep{3d-mcts}: License status unclear
    \item Autogrow4 \citep{autogrow4}: Apache License Version 2.0
\end{enumerate}
\paragraph{Chemical software}
\begin{enumerate}
    \item RDKit \citep{rdkit}: BSD 3-Clause License
    \item Vina \citep{vina}: Apache License Version 2.0
    \item Gnina \citep{gnina}: GPL-2.0 license
\end{enumerate}
\end{document}